\documentclass[11pt]{article}
\usepackage{amsmath, amssymb, amsthm}
\usepackage{fullpage}
\usepackage{hyperref}
\usepackage{graphicx}
\usepackage[english]{babel}
\usepackage{bm}
\usepackage{import}
\usepackage[utf8]{inputenc}
\usepackage[english]{babel}
\usepackage{amsmath}
\usepackage{verbatim}
\usepackage{bbm}
\usepackage{dsfont}
\usepackage{amsthm}
\usepackage{amssymb}
\usepackage{enumitem}
\usepackage{mathrsfs}
\usepackage{tikz}
\usepackage{tikzsymbols}
\usepackage{soul}
\usepackage[numbers]{natbib}
\usepackage{hyperref}
\usepackage{xcolor}
\usepackage{color}
\usepackage{caption,graphicx,newfloat}
\graphicspath{ {images/} }
\usepackage{booktabs} 
\usepackage[ruled]{algorithm2e} 
\usepackage{tcolorbox}
\usepackage{algorithmic}

\newtheorem{theorem}{Theorem}[section]
\newtheorem{lemma}[theorem]{Lemma}

\newtheorem{assumption}[theorem]{Assumption}
\newtheorem{claim}[theorem]{Claim}
\usepackage{tikz}
\usetikzlibrary{shapes,decorations.pathmorphing,decorations.markings,patterns,fadings,calc}

\usetikzlibrary{positioning,arrows.meta}

\usetikzlibrary{shapes,decorations.pathmorphing,decorations.markings,patterns,fadings,calc,arrows.meta,shapes.geometric}

\theoremstyle{definition}
\newtheorem{definition}{Definition}[section]

\newcommand{\reals}{{\mbox{\bf R}}}

\newcommand{\Tr}{\mathop{\bf Tr}}

\newcommand{\Expect}{\mathop{\bf E{}}}
\newcommand{\Var}{\mathop{\bf Var{}}}
\newcommand{\Cov}{\mathop{\bf Cov{}}}

\newcommand{\Prob}{\mathop{\bf Prob}}

\newcommand{\argmax}{\mathop{\rm argmax}}

\newcommand{\eg}{{\it e.g.}}
\newcommand{\ie}{{\it i.e.}}

\newcommand{\BEAS}{\begin{eqnarray*}}
\newcommand{\EEAS}{\end{eqnarray*}}
\newcommand{\BEA}{\begin{eqnarray}}
\newcommand{\EEA}{\end{eqnarray}}
\newcommand{\BEQ}{\begin{equation}}
\newcommand{\EEQ}{\end{equation}}
\newcommand{\BIT}{\begin{itemize}}
\newcommand{\EIT}{\end{itemize}}

\begin{document}

\title{Reasoning without Regret}
\author{Tarun Chitra\\ 
        Gauntlet \\
        \texttt{\small tarun@gauntlet.xyz} 
       }\date{\today}
\maketitle

\begin{abstract}
  Chain-of-thought reasoning enables large language models to solve multi-step tasks by framing problem solving as sequential decision problems. 
  Outcome-based rewards, which provide feedback only on final answers, show impressive success, but face challenges with credit assignment and slow convergence. 
  In contrast, procedure-based rewards offer efficient step-level feedback, but typically require costly human supervision. 
  We introduce \emph{Backwards Adaptive Reward Shaping} (BARS), a no-regret framework that converts sparse outcomes-based rewards into effective procedure-based signals. 
  BARS uses sparse rewards generated from terminal-state priors and cover trees to scale rewards while preventing exploitation. 
  With Bellman contraction and $(\Delta, \epsilon)$-gap rewards, our backward Euler solver achieves $\epsilon$-accuracy in $O\left((R_{\max}/\Delta)\log(1/\epsilon)\right)$ iterations with $O(\log T)$ dynamic regret over $T$ rounds. 
  Our analysis, based on generic chaining, continuous scaling limits, and non-linear Feynman-Kac bounds, connects recent outcome-based methods' empirical successes with the benefits of intermediate supervision.
  Combined, this provides the first rigorous no-regret algorithm for outcome reward shaping, providing a theoretical foundation for the empirical success of DeepSeek's R1.
\end{abstract}

\section{Introduction}
Reasoning language models (RLMs) enhance language models with capabilities that dramatically improve their performance on complex reasoning and planning tasks. 
Recent empirical work on chain-of-thought reasoning (CoT) has led to dramatic improvements in OpenAI's RLM performance in tasks such as Frontier Math, ARC, GSM8K and MATH~\cite{Wei2022, Trinh2024}; These capabilities draw on lessons from models that achieved superhuman performance in Poker~\cite{brown2019superhuman}, Go~\cite{silver2017mastering}, Diplomacy~\cite{meta2022human}, and other combinatorial hard games. 
RLMs involve large language models competing against each other to improve their prompts in a sequential manner.
This generates a chain-of-thought, a sequence of prompt-output pairs that not only leads to better answers for users but also provides interpretability into the model's reasoning, albeit with added inference time computation.

RLMs were first pioneered by~\cite{Wei2022}, which demonstrated that chain-of-thought reasoning endows models with explanatory power.
The allows for prompt planning, where a model can generate a sequence of prompts $P_1, \ldots, P_T$ that refine and explain the final output $O_T$ given to a user's initial prompt $P_0$. 
To train these models, a designer selects a reward schedule $r(P, O)$ to incentivize particular prompt-output pairs, then fixes a base language model $L$ and optimizes its cumulative reward via, for example, Bellman iteration.
This allows one to view CoT as a reinforcement learning problem in which language models and other neural networks act as black-box function calls to maximize cumulative reward and produce a high-quality final output $O_T$.

\paragraph{Closed and Open RLMs.}
OpenAI's closed-source O1 and O3 models pioneered commercial reasoning capabilities.
Their limited documentation has restricted research to empirical evaluations~\cite{jaech2024openai, ballon2025relationship, arrieta2025early}, hindering the theoretical understanding of CoT algorithms.\footnote{Author's note: There is some irony that this paper was written thanks to conversations with O3-mini}
Despite that, OpenAI's O1 and O3 have provided dramatic improvements in language model usage in research in disparate fields ranging from theoretical physics~\cite{yin2025exact} to healthcare~\cite{mondillo2025comparative,xu2025towards}.

Conversely, DeepSeek's R1 has emerged as the leading open weight RLM, achieving performance comparable to OpenAI's O1~\cite{guo2025deepseek}.
One key difference is that DeepSeek's model was trained with far fewer computational resources than O1 and claims to use fewer inference time compute resources~\cite{guo2025deepseek}.
The DeepSeek team has documented their approach extensively, highlighting their custom Group Relative Policy Optimization (GRPO) algorithm for CoT reasoning.
Their publications~\cite{guo2025deepseek, Guo2024} provide heuristics as to why GRPO outperformed methods like Monte Carlo Tree Search, which they specifically note~\cite[\S4]{deepseekai2025} failed to achieve strong results in R1 evaluations.
However, a theoretical understanding of DeepSeek's success is lacking.

\paragraph{Outcome and Procedure-Based Rewards.}
A key insight from the DeepSeek R1 paper is that designing rewards for different tasks is crucial for achieving human expert level performance.
The R1 model improves on DeepSeekMath's reward model~\cite{shao2024deepseekmath}, which was the first to use GRPO, and math-shepard~\cite{wang2023math}.
These reward models aim to avoid human-annotated training sets and instead learn optimal rewards.
This allows for training off of a much smaller annotated training set.

Outcome-based rewards provide a reward to a model that reaches a target output $O^{\star}$.
Generally, outcome-based rewards are easy to construct as one does not need a large tagged training set.
However, their success in solving reasoning tasks has been mixed~\cite{lightman2023let, choudhury2025process} as phenomena such as model reward hacking~\cite{skalse2022defining, miao2024inform} that worsen performance.

In contrast, procedure-based rewards give a sequence of rewards to the model if it is able to estimate a sequence of steps correctly.
For problems like math word problems, where the reasoning steps are well defined, this has been much more successful in practice~\cite{uesato2022solving, ouyang2022training, wang2023math}.
The main issue with procedure-based rewards is that they are usually constructed by humans (via RLHF) and are not scalable.
Although there have been many studies on reward models used to generate human-guided procedure-based rewards~\cite{moskovitz2023confronting, scheid2024optimal, wang2024secrets}, there has been little work on automatically learning optimal procedure-based rewards from outcome-based rewards.

\paragraph{Learning Procedure-Based Rewards from Sparse Outcome Rewards.}
Sparse rewards provide meaningful feedback only infrequently, typically after many actions or after achieving key goals.
In Chain-of-Thought reasoning, this approach effectively transforms outcome-based rewards into procedure-based rewards that attribute credit to individual reasoning steps.
This strategy has produced interpretable, high-performing models like DeepSeek's R1, showing that even with limited feedback, agents can master complex reasoning processes.

Recent work has formalized how sparse rewards enable one to derive procedure-based rewards from outcome-based rewards.
Empirical research by~\cite{setlur2024rewarding} shows that optimal procedural rewards can be estimated from outcome rewards through advantage optimization, while~\cite{jia2025we,zhou2025q} shows that procedure-reward learning can be achieved through regression with at most polynomial extra overhead.
Although not universally applicable~\cite{kim2025metastable}, the success of sparse rewards in reasoning models~\cite{setlur2024rewarding,deepseekai2025,qu2025optimizing} indicates that they are essential to create efficient and accurate CoTs while reducing model complexity.
We note that other attempts at using learning procedural rewards using, \eg~temporal difference methods (such as stepwise credit)~\cite{arjona2019rudder, zelikman2022star, pignatelli2023survey, shen2025satori, yao2023tree}.

\paragraph{RLMs and Regret.}
Given a fixed set of rewards, a natural question is how one should compare RLM reinforcement learning strategies.
Despite empirical success, little theoretical work has addressed CoT's convergence properties as an MDP.
Although numerous in-context learning studies exist (\eg,~\cite{krishnamurthy2024can, dong2022survey, zhang2023and, park2024llm, zhang2018dynamic, liu2023reason, ahn2023transformers}), none explain DeepSeek's observations about which reinforcement algorithms work empirically.
Regret minimization offers a natural framework for evaluating CoT performance by measuring how an online learner compares to an optimal offline algorithm, with sublinear regret guaranteeing convergence to coarse-correlated equilibria in game theory (\eg,~\cite{MasColellRegret}).

For CoT, such an equilibrium means that the prompt chain quickly reaches a neighborhood of the best one-shot output, yielding low regret for the end-user.
We show in~\S\ref{subsec:fwd-bellman} that continuous scaling limits of Bellman iterations provide an optimal comparison for procedure-based rewards, allowing for such regret to be well defined.
The regret of a CoT algorithm provides bounds on the reasoning steps needed to approach optimal rewards within a fixed computational budget $T$.

Previous studies~\cite{krishnamurthy2024can, park2024llm, liu2023reason, zhang2023and} examined the regret of transformer in-context learning relative to a function class that can be represented by a single language model.
However, DeepSeek's success suggests that RLM performance largely stems from reinforcement learning procedures optimizing rewards based on prompt sequences.
The dynamics of this process may involve multiple function classes whose complexities are based on reasoning tasks.
As such, CoT likely has a regret that differs substantially from that of an individual in-context learning task.

\paragraph{Continuous Limits of CoT processes.}
Regret bounds typically fix a function class and a time horizon $T$ and a bound algorithm performance over this period.
However, CoT queries have variable-length episodes depending on task difficulty, making comparison challenging across inputs and models.
To address this, we examine the continuous limits of CoT processes, where convergence rates provide a basis for comparing different CoT models using the same base LLM.

We represent the CoT process as a Markov Decision Process (MDP) whose trajectories (\ie~sequences of prompt-output pairs) obey a Bellman equation.
By taking a scaling limit of the Bellman equation, we embed MDP trajectories into a continuous process governed by a limiting Hamilton-Jacobi-Bellman (HJB) equation.
Solutions to this limiting equation represent the optimal value (rewards) along the paths between query and answer, serving as a proxy for optimal procedure-based rewards.
This, through a comparison principle~\cite{BarlesSouganidis1991,crandall1992user}, enables a comparison between forward iteration (\eg~policy gradient) and backwards iteration (\eg~value function iteration).
Continuous limits of MDPs have been studied in various learning settings, often relating MDPs to limit HJB equations through rescaled limits of discrete Bellman operators~\cite{bayraktar2023pde, drenska2023pde, drenska2020prediction, harvey2023optimal, bayraktar2025comparison, bayraktar2020malicious} or finding limiting ODEs or PDEs for online learning~\cite{zhang2022pde, mcmahan2014unconstrained, ivgi2023dog, orabona2016coin, cutkosky2018black, chen2022better}.

The solutions of the limiting HJB equation describe optimal procedure-based rewards by ensuring that the paths satisfy a maximal principle for the value $\overline{V}(\gamma)$.
We use viscosity solutions approximable by discrete Bellman equations~\cite{BarlesSouganidis1991} and dual to backward stochastic differential equations (BSDE).
This continuous limit implicitly specifies a BSDE for optimal process rewards, with our approximation accuracy determining how well we learn these rewards via Bellman iteration.
Using this framework, we define a hitting time (similar to the Snell Envelope~\cite[Ch. 9]{gobet2016monte}) to estimate the minimal reasoning steps needed to get $\epsilon$-close to optimal rewards.

\paragraph{This Paper.}
We describe CoT reasoning as a Markov Decision Process (MDP) with rewards $r(S_t, A_t)$.
Using this representation, we argue that the optimal value function gives a bound on procedure-based rewards (analogous to a claim in~\cite{jia2025we}).
Unlike~\cite{jia2025we}, we provide quantitative bounds on how well the optimal value can be approximated using a scaling limit of the CoT MDP.
In this continuous model, we study the limiting value function $\overline{V}(t, S_t)$, where $S_t$ is the state at time $t$.
This limiting value function satisfies a Hamilton-Jacobi-Bellman equation, analyzable via a BSDE representation.

Given a user query $P_0$, we study two iteration types: forward and backward.
The forward iteration generates a sequence of prompts $P_1, \ldots, P_T$ using a policy gradient, where the model chooses policies $\pi_t(\cdot | S_t)$ based on the context window $S_t$.
The backward iteration solves a dynamic program from the set of non-zero rewards $\mathrm{supp}(r)$ back to $P_0$.
If $\mathrm{supp}(r)$ can be described parsimoniously, one can solve for trajectories backward from rewards to $P_0$ using value function iteration~\cite{munos2008finite}, approximating the backward SDE associated with the HJB equation.

To determine whether forward or backward iteration better optimizes rewards within a fixed computational budget, we compare their hitting times to $\epsilon$-approximate the continuous scaling limit $\overline{V}$.
By coupling both forward Bellman iteration and backward Euler BSDE iteration to $\overline{V}$, we obtain hitting‐time bounds:
\begin{align*}
\tau^+_\epsilon = \Theta\!\left(\gamma_2(S,d)^2/\epsilon^2\right) &&
\tau^-_\epsilon = \Theta\!\left(\gamma_2(\mathrm{supp}(r),d)^2/\epsilon^2\right)
\end{align*}
where $\tau^+_{\epsilon}, \tau^{-}_{\epsilon}$ are the forward and backward hitting times.
The backward hitting time depends only on the Talagrand functional $\gamma_2(\mathrm{supp}(r), d)$, which can be much shorter than $\gamma_2(S, d)$ for sparse rewards.

Since policy suboptimality is at most twice the value‐error, these translate into regret bounds:
\begin{align*}
R_T = O\!\left((R_{\max}/\Delta)\log T\right) && 
R_T^{\mathrm{dyn}} = O(\log T)
\end{align*}
Based on these insights (see Fig.~\ref{fig:proof-steps}), we design Backwards Adaptive Reward Shaping (BARS), which samples potential answers from a prior distribution and selects optimal rewards for fast backwards iteration.
BARS uses online estimates of the Talagrand $\gamma_2$ functional (via, \eg,~cover trees~\cite{Beygelzimer2006}) to adapt the rewards per reasoning step, guaranteeing $\epsilon$-accuracy in $O(\log(1/\epsilon))$ backward steps with only $O(\log T)$ cumulative regret.
BARS can be seen as a method for adaptive reward shaping~\cite{ng1999policy}, which tunes the rewards to optimize the backward iteration.
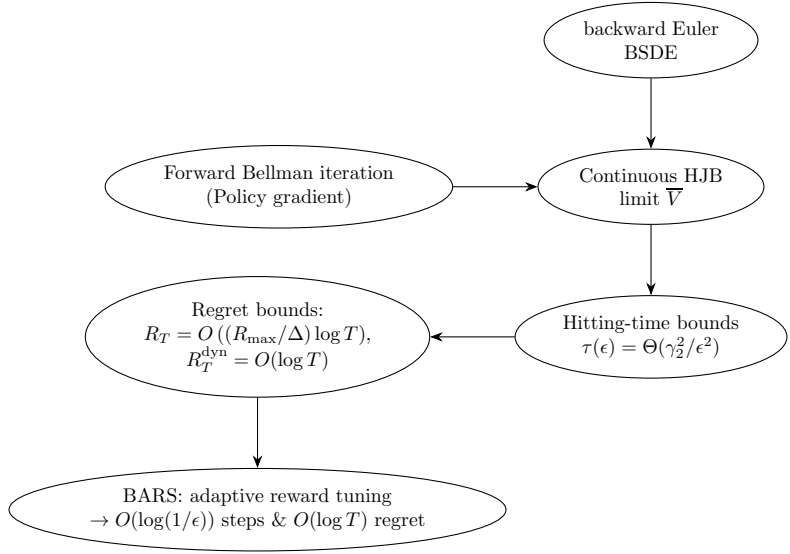
\begin{figure}
  \begin{center}
    \begin{tikzpicture}[
      scale=0.75, transform shape,
      >=Stealth,
      node distance=1.25cm and 1.5cm,
      bubble/.style={draw, ellipse, align=center, inner sep=4pt, font=\small}
    ]
      \node[bubble] (db) {Forward Bellman iteration \\ (Policy gradient)};
      \node[bubble] (hjb) [right=of db] {Continuous HJB\\limit $\overline V$};
      \node[bubble] (be) [above=of hjb] {backward Euler\\BSDE};
      \node[bubble] (hit) [below=of hjb] {Hitting‑time bounds\\$\tau(\epsilon)=\Theta(\gamma_2^2/\epsilon^2)$};
      \node[bubble] (reg) [left=of hit] {Regret bounds:\\$R_T=O\left((R_{\max}/\Delta)\log T\right)$,\\$R_T^{\mathrm{dyn}}=O(\log T)$};
      \node[bubble] (BARS) [below=of reg] {BARS: adaptive reward tuning\\$\to O(\log(1/\epsilon))$ steps \& $O(\log T)$ regret};
    
      \draw[->] (db) -- (hjb);
      \draw[->] (be) -- (hjb);
      \draw[->] (hjb) -- (hit);
      \draw[->] (hit) -- (reg);
      \draw[->] (reg) -- (BARS);
    \end{tikzpicture}
    \end{center}
\caption{Flowchart of the results proved in in this paper}
\label{fig:proof-steps}
\end{figure}

\paragraph{Notation.}
\begin{itemize}
    \item $f(n) = O(g(n))$ means there exists $C, M > 0$ such that $\forall n > M, f(n) \leq C g(n)$
    \item $f(n) = \Omega(g(n))$ means there exists $C, M > 0$ such that $\forall n > M, f(n) \geq C g(n)$
    \item $f(n) = \Theta(g(n))$ means there exists $C, C', M > 0$ such that $\forall n > M, f(n) \leq C g(n)$ and $g(n) \leq C' f(n)$
    \item We use $f(n) \asymp g(n)$ interchangeably with $f(n) = \Theta(g(n))$
    \item For a metric space $(X, d)$, we define $B(x, r) = \{y \in X : d(x,y) \leq r\}$ denotes the ball of radius $r$ centered at $x$ in the metric space $(X, d)$
    \item For $p \geq 1$, $\|x\|_p = \left(\sum_{i=1}^n |x_i|^p\right)^{1/p}$ denotes the $L^p$ norm of vector $x \in \reals^n$
    \item $\mathrm{supp}(r) = \{s \in S : r(s) \neq 0\}$ denotes the support of reward function $r$
    \item $\overline{V}$: Overlined variables will refer to continuous limits of discrete processes
    \item For a finite state space $S$, we will denote discrete operators $\mathcal{T} : \reals^S \rightarrow \reals^S$ via $\mathcal{T}[\phi](s)$.
    If these operators have continuous limits, they will be denoted by $\overline{\mathcal{T}}$.
\end{itemize}

\section{Markov Decision Processes and their Scaling Limits}
A \emph{Markov Decision Process (MDP)} is defined by the tuple $(S, A, P, r, \gamma)$, where $S$ is the state space, $A$ is the action space, $P(s' \mid s,a)$ denotes the transition probabilities, $r(s)$ is the reward function, and $\gamma\in [0,1)$ is the discount factor.
In this paper, we will only be concerned with non-negative rewards, \ie~$r(s) \geq 0$.
Moreover, we assume that all state and action spaces are compact and simply connected.

A \emph{policy} $\pi: S \rightarrow \Delta(A)$ is a mapping from states to distributions over actions, where $\pi(a|s)$ denotes the probability of taking action $a$ in state $s$.
A \emph{trajectory} $\tau = (s_0, a_0, s_1, a_1, \ldots, s_T, a_T)$ is a sequence of states and actions sampled according to $s_0 \sim \mu_0$, $a_t \sim \pi(\cdot|s_t)$, and $s_{t+1} \sim P(\cdot|s_t, a_t)$.
The \emph{cumulative reward} along a trajectory is $R(\tau) = \sum_{t=0}^{T} \gamma^t r(s_t, a_t)$.
The \emph{discounted Bellman operator} is 
\begin{equation}\label{eq:discounted-bellman}
\mathcal{T}^\pi[V](s) = r(s, \pi(s)) + \gamma \sum_{s'} P(s'|s, \pi(s)) V(s')
\end{equation}
An MDP induces a \emph{state distribution} $d^\pi(s) = (1-\gamma) \sum_{t=0}^{\infty} \gamma^t \Prob(s_t = s | \pi)$.
The \emph{expected value function} is $J(\pi) = \Expect_{s \sim d^\pi}[V^\pi(s)] = \Expect_{\tau \sim \pi}[R(\tau)]$.
Let $J^{\pi^{\star}} = \max_{\pi} J(\pi)$ denote the optimal expected value.
The \emph{regret} of a policy $\pi$ is $\text{Regret}(\pi) = J(\pi^{\star}) - J(\pi)$, where $\pi^{\star} \in \argmax_{\pi} J(\pi)$, measuring the gap between the optimal and the achieved rewards.

\paragraph{Policy Gradients.}
Policy gradient methods directly optimize a parameterized policy $\pi_\theta$ by ascending the gradient of the expected return $V(\pi_\theta)$ with respect to policy parameters $\theta$ \cite{SchulmanEtAl2017}. 
These methods compute updates of the form:
\[
\nabla_\theta J(\pi_\theta) \approx \Expect_{s\sim d_{\pi_\theta},\,a\sim\pi_\theta}\left[\nabla_\theta\log\pi_\theta(a|s)A^{\pi_\theta}(s,a)\right],
\]
where $A^{\pi_\theta}(s,a) = Q^{\pi_\theta}(s,a) - V^{\pi_\theta}(s)$ is the advantage function. 
This formulation uses advantages due to the performance difference lemma~\cite{kakade2002approximately}, which states:
\[
V^{\pi'}(s) - V^{\pi}(s) = \Expect_{s'\sim d_{\pi'},\,a\sim\pi'}\left[A^{\pi}(s',a)\right]
\]
This lemma provides the theoretical foundation for policy improvement, showing that a positive expected advantage guarantees improved performance.

Modern methods like Trust Region Policy Optimization (TRPO)~\cite{Schulman15} and Proximal Policy Optimization (PPO)~\cite{SchulmanEtAl2017} build on this by constraining policy updates for stable learning. 
TRPO enforces a KL divergence constraint between consecutive policies: $D_{KL}(\pi_{\theta_{t+1}} \| \pi_{\theta_t}) \leq \delta$, while PPO uses a clipped surrogate objective that penalizes large policy changes. 
By controlling the step size $\Delta\theta_t$, these methods effectively create a contraction mapping in policy space:
\[
\langle \nabla V^{\pi_t}(s),\Delta\theta_t\rangle\le -m\,\left|V^{\pi_t}(s)-V^{\pi^*}(s)\right|
\]
This controlled optimization is crucial for convergence guarantees and connects to our viscosity solution framework for continuous control problems.

\paragraph{Chain of Thought Reasoning as an MDP.}
Chain-of-Thought (CoT) reasoning enables language models to solve complex problems through step-by-step reasoning.
We model the LLM as a black-box oracle $L$ mapping prompts to outputs ($O_t = L(P_t)$), focusing on the sequential decision process rather than model architecture.

In our MDP formulation, actions are possible prompts to the LLM, represented as matrices in $\reals^{m \times d}$ (where $m$ is token count and $d$ is embedding dimension).
States comprise the context window containing the prompt-output history, represented as vectors in $\reals^{m \times d \times T}$ (where $T$ is the length of the context, giving the dimension of the context $C = m \times d \times T$).
State transitions occur when a prompt-output pair $(P_t, O_t)$ is appended to the current context: $S_t = (P_0, O_0) \oplus \cdots \oplus (P_t, O_t)$, with $\oplus$ denoting concatenation.

This naturally forms an MDP where states represent accumulated context, transitions follow from the LLM's responses, and rewards reflect reasoning quality or solution correctness.
Most contemporary reasoning models implement this using policy gradient methods (\eg, DeepSeek's GRPO~\cite{deepseekai2025}).

\subsection{Reward Structure and Computational Trade-offs}
The reward function $r(s, a)$ in CoT reasoning involves balancing exploration and exploitation, similar to multi-armed bandit problems~\cite{Srinivas2010}.
This challenge is magnified in CoT's high-dimensional space where each token choice is a decision point, making exhaustive exploration infeasible.

CoT systems often employ \emph{sparse rewards} --- non-zero only on a small subset of states --- as a computational requirement rather than a design choice.
Dense reward evaluation is often prohibitively expensive given language model inference costs.
Sparse rewards typically appear as outcome-based evaluations~\cite{setlur2024rewarding}, such as rewards for correctly solved problems or functioning generated code.

As we show in~\S\ref{subsec:fwd-bellman}, value function solutions to the Bellman equation (eq.~\ref{eq:discounted-bellman}) transform these sparse rewards into procedure-based rewards by propagating the value backward from the terminal states of high reward~\cite{sutton1998reinforcement}.
This creates a dense reward landscape from initially sparse rewards, similar to procedure-based rewards in RLHF systems~\cite{qu2024recursive, kim2025metastable, miao2024inform, moskovitz2023confronting, scheid2024optimal,wang2024secrets}, enabling efficient learning in spaces where direct reward signals are rare.

\paragraph{Formal Definitions of Reward Structures.}
We define rewards as \emph{sparse} if they are non-zero only on a subset $S' \subset S$ where $|S'| \ll |S|$.
For infinite state spaces, we use covering numbers to measure relative size (see~\S\ref{sec:covering-numbers}).
Typically, $S'$ represents terminal states or goal conditions, with $r(s) = 0$ for all $s \notin S'$.
We denote the support of a reward function by $\mathrm{supp}(r) = \{s \in S : r(s) \neq 0\}$.

Rewards are \emph{$(\Delta, \epsilon)$-gapped} if for any state $s$, when comparing the optimal action $a^*(s) = \arg\max_a r(s,a)$ with any suboptimal action $a$ where $d(a, a^*) > \epsilon$, we have:
\begin{equation}\label{eq:gap-condition}
r(s, a^*) - r(s, a) > \Delta
\end{equation}
This property connects to bandit problems, enabling logarithmic regret bounds~\cite{Srinivas2010,LaiRobbins1985,kleinberg2008multi}.
In CoT reasoning, gapped rewards create clearer learning signals by establishing a meaningful separation between optimal and suboptimal reasoning paths, helping induce procedure-based reward structures~\cite{dong2024rlhf, wang2024secrets, shen2024improving}.


\subsection{Forward Bellman Value Iteration and Limiting PDEs}\label{subsec:fwd-bellman} 
Forward Bellman value iteration iteratively applies the Bellman operator to approximate the optimal value function. 
Starting with an arbitrary value function $V_0$, the algorithm updates the value function according to $V_{k+1} = \mathcal{T}[V_k]$, where $\mathcal{T}$ is the optimal Bellman operator defined as $\mathcal{T}[V](s) = \max_a [r(s,a) + \gamma \sum_{s'} P(s'|s,a) V(s')]$. 
The optimal value function $V^*$ is the unique fixed point of this operator, satisfying $V^* = \mathcal{T}[V^*]$. 
The convergence rate of value iteration is geometric, with $\|V_k - V^*\|_\infty \leq \gamma^k \|V_0 - V^*\|_\infty$, which follows from the fact that the Bellman operator is a contraction mapping in the infinity norm~\cite[Prop. 4.4]{bellemare2023distributional}.

The Bellman residual, defined as $\|\mathcal{T}[V] - V\|$, provides a measure of how far a value function is from optimality. 
Bellman residual minimization methods directly optimize this quantity to find approximate value functions. 
These approaches, first studied in~\cite{munos2008finite}, provide finite-time bounds on the approximation error and can be particularly effective when dealing with continuous state spaces or function approximation settings where exact value iteration is intractable.
In this section, we look at the scaling limits of Bellman Value Iteration and describe the mean-field value function.

\paragraph{Controlled Diffusion Limit.}
One can view the Bellman value iteration as a discretization of a controlled diffusion process~\cite{krylov2002introduction}.
Under regularity conditions on the kernel $P(s' | s, a)$ one can compute a drift $b(S_t, a_t)$ and a covariance $\sigma(S_t, a_t)$ such that
\begin{align*}
\Expect_P[s' - s\mid s,a]=b(s,a)\,\Delta t+o(\Delta t)
&&
\Cov_P[s'-s\mid s,a]=\sigma\sigma^T(s,a)\,\Delta t+o(\Delta t)
\end{align*}
where $\Expect_P$ and $\Cov_P$ are the expectation and covariance, respectively, under the distribution induced by $P$. 
We will provide a brief description of how to formalize this limit, but refer the interested reader to~\cite{krylov2002introduction, kushner2001viscosity} for more details.

Let $\Omega^{\delta}$ be a discretization of $S \times A$ into a grid of size $\delta$.
Given that $S \times A$ is compact and simply connected, such a discretization exists and has a single component.
For each point $(s, a) \in S \times A$, we define $(s_{\delta}, a_{\delta})$ to be the nearest point in $\Omega^{\delta}$ with arbitrary tie-breaking.
We define the discretized transition kernel $P^{\delta}(s' | s, a) = P(s'_{\delta} | s_{\delta}, a_{\delta})$, which effectively rounds the continuous kernel to the grid $\Omega^{\delta}$.
We now define two quantities:
\begin{align*}
b^{\delta}(s, a) &= \frac{1}{\delta} \sum_{s'} P^{\delta}(s' | s, a) (r(s') - r(s)) &
\sigma^{\delta}(s, a) &= \frac{1}{\delta} \sum_{s'} P^{\delta}(s' | s, a) (s' - s)(s' - s)^T
\end{align*}
Let $(s_t^{\delta}, a_t^{\delta})$ be a trajectory sampled from $P^{\delta}$.
Theorem 3.1 of~\cite{kushner2001viscosity} states that if the moments of this trajectory satisfy $\sup_{\delta \geq 0} \sup_{p} \Expect_{P^{\delta}}[\Vert (s_t^{\delta}, a_t^{\delta}) \Vert^p] < \infty$, then as $\delta \to 0$, $b^{\delta}(s, a) \rightarrow b(s, a), \sigma^{\delta}(s, a) \rightarrow \sigma(s, a)$.
Moreover, the trajectory $(s_t^{\delta}, a_t^{\delta})$ weakly converges to a solution of the stochastic differential equation
\begin{equation}\label{eq:controlled-diffusion}
dS_t = b(S_t,a_t)\,dt + \sigma(S_t,a_t)\,dW_t
\end{equation}
where $b, \sigma$ are Lipschitz functions.
This limit effectively demonstrates that (forward) value iteration for an MDP samples a limiting diffusion process under mild assumptions, which we formalize below.
\begin{assumption}\label{assm:forward-assumption}
The transition kernel $P$ generates sample trajectories with finite $p$-th moments for all $p \geq 1$.
The limiting drift $b$ and covariance $\sigma$ are Lipschitz functions.
\end{assumption}

\paragraph{Hamilton-Jacobi-Bellman Equation.}
For the controlled diffusion~\eqref{eq:controlled-diffusion}, we can define a limiting value function $\overline{V}(s)$:
\[
\overline{V}(s) = \max_{\{a_t\}} \Expect_{S_t}\left[ \int_{0}^{\infty} e^{-\gamma t} r(S_t, a_t) \big\vert S_0 = s\right]
\]
where the expectation is over the continuous trajectory.
This value function $\overline{V}$ satisfies the Hamilton-Jacobi-Bellman (HJB) equation
\begin{equation}\label{eq:hjb}
\gamma \overline{V}(s) = \sup_{a \in A} \left(r(s, a) + b(s, a) \cdot \nabla \overline{V}(s) + \frac{1}{2} \Tr \sigma(s, a)\sigma(s, a)^T \nabla^2 \overline{V}(s) \right)
\end{equation}
where $b, \sigma$ are, respectively, the drift and covariance of the controlled diffusion process.
We define the limit Bellman operator $\overline{\mathcal{T}}$ corresponding to this HJB equation:
\begin{equation*}
(\overline{\mathcal{T}} \overline{V})(s) = \frac{1}{\gamma}\sup_{a \in A} \left(r(s, a) + b(s, a) \cdot \nabla \overline{V}(s) + \frac{1}{2} \Tr \sigma(s, a)\sigma(s, a)^T \nabla^2 \overline{V}(s) \right)
\end{equation*}
This operator maps a function $\overline{V}$ to a new function $\overline{\mathcal{T}} \overline{V}$, and the HJB equation states that $\gamma \overline{V} = \overline{\mathcal{T}}\overline{V}$, or equivalently, $\overline{V} = \frac{1}{\gamma}\overline{\mathcal{T}} \overline{V}$.

\paragraph{Viscosity Solutions to PDEs.}
This equation is the continuous limit of the discrete Bellman operator~\eqref{eq:discounted-bellman} (see~\cite[Ch. 2]{krylov2002introduction}).
Since the HJB equation~\eqref{eq:hjb} is a non-linear PDE, it may lack smooth solutions, which is why we will need viscosity solutions.

\begin{definition}
A continuous function $\overline{V}: S \rightarrow \reals$ is a \textit{viscosity solution} to~\eqref{eq:hjb} if:
\begin{itemize}
    \item (Subsolution) For any $\phi \in C^2(S)$ where $\overline{V} - \phi$ has a local maximum at $s_0$:
    \[
        \gamma \overline{V}(s_0) \leq \sup_{a \in A} \left(r(s_0, a) + b(s_0, a) \cdot \nabla \phi(s_0) + \frac{1}{2} \Tr \sigma(s_0, a)\sigma(s_0, a)^T \nabla^2 \phi(s_0) \right)
    \]
    \item (Supersolution) For any $\phi \in C^2(S)$ where $\overline{V} - \phi$ has a local minimum at $s_0$:
    \[
        \gamma \overline{V}(s_0) \geq \sup_{a \in A} \left(r(s_0, a) + b(s_0, a) \cdot \nabla \phi(s_0) + \frac{1}{2} \Tr \sigma(s_0, a)\sigma(s_0, a)^T \nabla^2 \phi(s_0) \right)
    \]
\end{itemize}
\end{definition}

\noindent Viscosity solutions use supersolutions and subsolutions to replace derivatives with smooth test functions at the local extrema.
Viscosity solutions handle nondifferentiable points in the value function while providing stability against perturbations through inequalities at local extrema.
Under mild conditions, they form the correct limit of discrete Bellman iterations, making them the appropriate framework for analyzing convergence in our setting.

\paragraph{Convergence of Discrete Bellman Iteration.}
We define the policy-dependent Bellman operator $\mathcal{T}^{\pi, \delta}$ using transition kernel $P^{\delta}$:
\begin{align*}
\mathcal{T}^{\pi, \delta}[V](s) &= r(s, \pi(s)) + \frac{1}{\gamma} \sum_{s'} P^{\delta}(s' | s, \pi(s)) V(s')
\end{align*}
The Bellman optimality operator is then defined as: $\mathcal{T}^{\delta}[V](s) = \max_{\pi} \mathcal{T}^{\pi, \delta}[V](s)$.
We will first heuristically show that as $\delta \to 0$, $\mathcal{T}^{\delta}$ converges to the continuous operator $\overline{\mathcal{T}}$ of the HJB equation when the value function is smooth.
Using the definitions of $b^{\delta}$ and $\sigma^{\delta}$ from the previous section:
\begin{align*}
\sum_{s'} P^{\delta}(s' | s, a)(s' - s) &= \delta b^{\delta}(s,a) + o(\delta) \\
\sum_{s'} P^{\delta}(s' | s, a)(s' - s)(s' - s)^T &= \delta \sigma^{\delta}(s,a)\sigma^{\delta}(s,a)^T + o(\delta)
\end{align*}
A Taylor expansion of $V(s')$ around $s$ demonstrates the approximation of~\eqref{eq:hjb}:
\begin{align*}
\mathcal{T}^{\pi, \delta}[V](s) &= r(s, \pi(s)) + \frac{V(s)}{\gamma} + \frac{\delta}{\gamma} b^{\delta}(s, \pi(s)) \cdot \nabla V(s) \\
&\quad + \frac{\delta}{2\gamma} \text{Tr}(\sigma^{\delta}(s, \pi(s))\sigma^{\delta}(s, \pi(s))^T \nabla^2 V(s)) + o(\delta)
\end{align*}
While smooth solutions to~\eqref{eq:hjb} may not always exist, viscosity solutions do. 
The Barles-Souganidis theorem~\cite{BarlesSouganidis1991} guarantees the convergence of discrete Bellman iterations
to these viscosity solutions under conditions of consistency, stability, and monotonicity. 
We prove the following result in Appendix~\ref{app:barles}:
\begin{claim}\label{claim:barles}
Suppose that we have rewards $r$ that are bounded, $(\delta, \epsilon)$-gapped, and non-negative.
Then the trajectories of $\mathcal{T}^{\pi, \delta}[V](s)$ converge to viscosity solutions of $\overline{\mathcal{T}V}(s)$
\end{claim}

\paragraph{Why is the solution to the HJB equation the optimal procedure-based rewards?}
In the discrete setting, let \(V^\delta\) be the unique fixed point of the Bellman operator \(T^\delta\) on a mesh of resolution \(\delta\). 
Under the Barles–Souganidis conditions of monotonicity, stability, and consistency (Appendix~\ref{app:barles}), one shows
\[
\|V^\delta - \overline V\|_\infty \;\xrightarrow[\delta\to0]{}\;0
\]
where \(\overline V\) is the viscosity solution of the HJB PDE (eq.~\eqref{eq:hjb}) and its discrete schemes converge uniformly.
Denote the greedy discrete policy
\[
\pi^\delta(s)\;=\;\arg\max_{a\in A}\left\{r(s,a)+\gamma\sum_{s'}P^\delta(s'\mid s,a)\,V^\delta(s')\right\}
\]
Suppose that the rewards satisfy a \((\Delta,\epsilon)\)-gap—so that any suboptimal action is at least \(\Delta\) worse.
Then once \(\|V^\delta-\overline V\|_\infty<\tfrac\Delta2\), the discrete argmax cannot flip and \(\pi^\delta=\pi^*\) everywhere.

More generally, even without a gap, a uniform sup‑norm error \(\|V^\delta-\overline V\|_\infty\le\epsilon\) controls policy regret. 
By the performance‑difference lemma~\cite{kakade2002approximately},
\[
J(\pi^*) - J(\pi^\delta)
\;=\;
\Expect_{s\sim d_{\pi^*}}\left[\overline V(s)-V^\delta(s)\right]
\;+\;\Expect_{s\sim d_{\pi^\delta}}\left[V^\delta(s)-\overline V(s)\right]
\;\le\;2\,\epsilon
\]
Thus uniform convergence of the discrete solver to the continuous limit \(\overline V\) immediately yields a \(2\epsilon\) bound on policy suboptimality, justifying the use of \(\overline V\) as the correct target for learning optimal procedure‑based rewards.

\paragraph{Hitting Time for Forward Equations.}
In order to quantify when a discrete solution can approximate the viscosity solution up to an error $\epsilon$, we consider hitting times.
We define the \emph{hitting time} for the forward iteration as the number of steps needed to reach a desired accuracy. 
Let
\[
\tau_\epsilon^+ = \min\left\{ k: \|V^\delta_k - \overline{V}\|_{L^\infty(S)} \le \epsilon \right\}
\]
be the first forward iteration in which the sup-norm error between the discrete Bellman solution and the true value function falls below a prescribed tolerance $\epsilon$. 
This notion serves as a measure of the computational effort required for the forward iteration to propagate the value information efficiently through the state space.

\subsection{Backwards Bellman Value Iteration}
We develop the backward propagation framework for the value function by considering the nonlinear Feynman--Kac representation of the Hamilton--Jacobi--Bellman (HJB) equation. 
In contrast to the forward Bellman iteration, which propagates the value from the query forward, the backward approach exploits the terminal reward and iteratively ``pulls'' this information back toward the initial state. 
This backward procedure is naturally cast in the framework of backward stochastic differential equations (BSDEs).

\paragraph{Nonlinear Feynman--Kac Representation.}  
Let $(S,A,P,b,\sigma,\gamma)$ be the controlled diffusion limit of our MDP as defined in the previous section.
Under Assumption~\ref{assm:forward-assumption}, the value function that solves~\eqref{eq:hjb} is the unique bounded, uniformly continuous viscosity solution.
By the nonlinear Feynman--Kac formula (\eg~\cite{ElKarouiPengQuenez1997,PardouxPeng1990}), for any terminal time $T$ and terminal reward $g(X_T)$, the process $(Y_t,Z_t)$ solving the backward SDE  
\begin{equation*} 
Y_t = g(X_T) + \int_t^T f\left(X_s,Y_s,Z_s,\Gamma_s\right)\,ds - \int_t^T Z_s\cdot dW_s
\end{equation*}  
with $Z_s$ and $\Gamma_s$ denoting the first and second‐order viscosity terms, respectively (see~\cite[\S2]{crandall1992user}), recovers $Y_t = \overline{V}(t,X_t)$.
One can think of the first and second-order viscosity terms as upper and/or lower bounds on the gradient and hessian, much like the subgradient gives a bound on possible gradient steps.
Finally, we note that the full driver (\ie~transition function) is: 
\begin{equation*} 
f(x,y,z,\Gamma) = \sup_{a\in A}\left\{\,r(x,a) + b(x,a)\cdot z + \tfrac12\,\Tr\left[\sigma(x,a)\sigma(x,a)^{T}\,\Gamma\right]\right\} - \gamma\,y
\end{equation*}  
One can view this as continuous analogue of the transition kernel and policy for updating state evolution.
Interested readers can find more details in~\cite{ElKarouiPengQuenez1997,PardouxPeng1990,peng1992nonlinear}

\paragraph{Numerical Approximation of Backward SDEs.}  
Discretize the interval $[0,T]$ by $0=t_0<t_1<\cdots<t_N=T$ with a uniform step $\delta=T/N$.
Given an Euler--Maruyama approximation~\cite[\S5]{gobet2016monte}  
\[
   X_{k+1} = X_k + b(X_k,a_k)\,\delta + \sigma(X_k,a_k)\,\Delta W_{k+1}
\]  
we define the backward Euler operator on a grid function $V: S^{\delta} \rightarrow \reals$ by  
\begin{equation*} 
\mathcal{T}_{\mathrm{BSDE}}^{\delta}[V](s) = \Expect_{\Delta W}\left[\,V\left(X_{k+1}\right) + \delta\,f\left(s,\,V(X_{k+1}), \Expect_{\Delta W_{k+1}}[V(X_{k+1})\,\Delta W_{k+1}]/\delta,\,\Gamma\right)\right]
\end{equation*}  
Iterating backward with terminal $Y^\delta_N = g$  
\begin{align*}
Y^\delta_k = \mathcal{T}_{\mathrm{BSDE}}^{\delta}[Y^\delta_{k+1}] && k = N-1,\ldots,0
\end{align*}
yields a monotone, stable scheme converging to the viscosity solution of (\ref{eq:hjb}) as $\delta\to0$~\cite{BarlesSouganidis1991,PardouxPeng1990}.  

\paragraph{Hitting Time for Backwards Equations.}
We define \emph{hitting time} for the backward iteration analogously to the forward case. 
Let
\[
\tau_\epsilon^- = \min\left\{ k: \|Y^\delta_k - \overline{Y}_k\|_{L^\infty(S)} \le \epsilon \right\},
\]
be the first backward iteration when the sup–norm error between the discrete BSDE solution and the true value function falls below a prescribed tolerance $\epsilon$. 
This notion serves as a measure of the computational effort required for the backward iteration to propagate the terminal reward information efficiently through the state space.

\paragraph{Exploitation versus Exploration.}
Forward Bellman iteration requires extensive exploration due to initially weak gradient information (\ie~there are no rewards near the initial query), with diffusion dominating and causing wide trajectory exploration similar to multi-armed bandit problems.
However, backward iteration starts from strong terminal reward signals, exploiting this information to pull value estimates backward with less exploration, potentially achieving lower hitting times when terminal conditions are accurately approximated.
There is a critical balance (which we quantify in~\S\ref{sec:results}): higher backward rewards can intensify propagation and encourage exploration to reconcile large gradients, while excessively low rewards risk getting stuck in suboptimal regions, making appropriate reward scaling essential for effectively balancing exploitation and exploration.

\paragraph{Comparison of Forward and Backward Discrete Solutions.}
Finally, we observe that discretization errors differ between schemes: forward iteration accumulates error through diffusion-dominated propagation, while backward iteration anchors approximation with strong terminal signals.
Backward iteration achieves lower error when the terminal conditions are well approximated in terms of grid error $\delta$ and metric complexity $\gamma_2(\mathrm{supp}(r), d)$. 
Both approaches incur local errors of the order $\delta$, but backward updates yield faster convergence due to stronger exploitation, 
highlighting the importance of proper reward scaling with metric complexity.

\section{Covering Numbers and Generic Chaining}\label{sec:covering-numbers}
In order to estimate a rate of convergence for a discrete Bellman operator to its scaling limit, one needs to consider covering numbers.
Covering numbers provide a fundamental measure of the ``size'' of a space in terms of its approximability by finite sets.

\paragraph{Covering Numbers.}
For a metric space $(T,d)$, the $\epsilon$-covering number $\mathcal{N}(T,d,\epsilon)$ is the minimum number of balls of radius $\epsilon$ needed to cover $T$. 
Formally, $\mathcal{N}(T,d,\epsilon) = \min\{n : \exists t_1,\ldots,t_n \in T \text{ such that } T \subseteq \cup_{i=1}^n B(t_i,\epsilon)\}$, where $B(t_i,\epsilon) = \{t \in T : d(t,t_i) \leq \epsilon\}$ is the closed ball of radius $\epsilon$ centered on $t_i$.
For a class of functions $\mathcal{F}$ in the domain $X$ with metric $d$, the covering number $\mathcal{N}(\mathcal{F},d,\epsilon)$ represents the minimum number of functions needed to approximate any $f \in \mathcal{F}$ within $\epsilon$ accuracy under metric $d$. 
Formally:
\[
\mathcal{N}(\mathcal{F},d,\epsilon) = \min\left\{n : \exists f_1,\ldots,f_n \in \mathcal{F} \text{ such that } \min_{i \leq n} d(f,f_i) \leq \epsilon \text{ for all } f \in \mathcal{F}\right\}
\]

Covering numbers capture how efficiently a continuous space can be discretized, providing bounds on sample complexity and generalization error in learning theory. 
However, they often yield suboptimal bounds because they treat all scales equally.

\paragraph{Talagrand's $\gamma_2$ Functional.} A more refined complexity measure is Talagrand's $\gamma_2$ functional, which accounts for the multi-scale nature of approximation. 
For a metric space $(T,d)$, $\gamma_2(T,d)$ is defined as:
\[
\gamma_2(T,d) = \inf_{(T_n)} \sup_{t \in T} \sum_{n=0}^{\infty} 2^{n/2} \cdot d(t, T_n)
\]
where the infimum is taken over all sequences of sets $T_n \subset T$ with $|T_0| = 1$ and $|T_n| \leq 2^{2^n}$ for $n \geq 1$, and $d(t,T_n) = \inf_{s \in T_n} d(t,s)$.
One reason for using $\gamma_2$ is that it is a more refined complexity measure and often provides tight upper and lower bounds.
For more details, see~\cite{Talagrand2005, nelson2016chaining}.

\paragraph{Properties of $\gamma_2$.}
The $\gamma_2$ functional satisfies several useful properties \cite{Talagrand2005}:
\begin{enumerate}
  \item For Cartesian products: $\gamma_2(S \times T, d_S \oplus d_T) \asymp \gamma_2(S,d_S) + \gamma_2(T,d_T)$
  \item For Lipschitz functions: $\gamma_2(\text{Lip}_L(X), \|\cdot\|_{\infty}) \asymp \gamma_2(X,d) + \log L$, where $\text{Lip}_L(X)$ is the class of $L$-Lipschitz functions on $X$
  \item For finite-dimensional spaces: $\gamma_2(B_d, \|\cdot\|_2) \asymp \sqrt{d}$, where $B_d$ is the unit ball in $\reals^d$
  \item Dudley's integral bound: $\gamma_2(T,d) \lesssim \int_0^{\text{diam}(T)} \sqrt{\log \mathcal{N}(T,d,\epsilon)} \, d\epsilon$
  \item Majorizing measures theorem: $\gamma_2(T,d) \asymp \Expect\sup_{t \in T} X_t$, where $(X_t)_{t \in T}$ is a centered Gaussian process with $\Expect(X_s - X_t)^2 = d^2(s,t)$, establishing the tight connection between $\gamma_2$ and the expected supremum of Gaussian processes
\end{enumerate}

\paragraph{Complexity and Applications to Language Models.}
The $\gamma_2$ functional effectively analyzes continuous spaces in reinforcement learning, with Dudley's integral bound connecting to algorithms and majorizing measures providing tight learning bounds.
Spaces with low $\gamma_2$ values allow for an efficient approximation: $\gamma_2(T,d) \asymp \sqrt{d}$ exhibits the dimensionality curse, while $\gamma_2(T,d) \asymp \log d$ often allows dimension-independent rates.
For language models, $\gamma_2$ quantifies complexity through various established frameworks: Bartlett et al.~\cite{bartlett2017spectrally} derived generalization bounds, Golowich et al.~\cite{golowich2018size} used Frobenius norms, and Maurer et al.~\cite{maurer2016benefit} connected sample complexity to representation spaces.
In our setting with prompts as matrices in $\reals^{m \times d}$ and context windows as tensors in $\reals^{m \times d \times T}$, $\gamma_2$ naturally quantifies complexity and yields tight bounds on reinforcement learning performance.

\subsection{Bellman Iteration Convergence Rates}\label{sec:convergence-rates}
We will show how covering numbers and the Talagrand functional are closely related to how fast Bellman iteration converges.
In particular, we will show that the level of discretization of the transition kernel needed for Bellman iteration to be a $\epsilon$ close solution of the HJB equation requires a covering of radius $\Omega(\gamma_2^2(S, d))$.
As we shall see in the subsequent section, this implies provides bounds on how fast the policy-induced random walk converges to the HJB solution.

\paragraph{Controlled Diffusion.}
Suppose that we estimate a value function from a controlled diffusion process, eq.~\eqref{eq:controlled-diffusion}, and that we have a $\delta$-approximate solution $V^{\delta}(s)$ to the value function for $\delta > 0$
We want to find the maximum $\delta$ such that $\max_{s \in S} |V^{\delta}(s) - V(s)| \leq \epsilon$:
\begin{claim}\label{claim:chaining-error}
 Suppose that for each \(\delta>0\) and any \(\delta\)-net $S_{\delta}$, we have a discrete value function \(V^\delta\) satisfying \(\displaystyle\sup_{s_i\in\mathcal S_\delta}|V^\delta(s_i)-V(s_i)|\le C\,\delta\) with both \(V\) and \(V^\delta\) are \(L\)–Lipschitz on \(S\).
  \[
  \sup_{s\in S}\left|V^{\delta}(s)-V(s)\right| \leq C\delta + 2 L\sqrt{\delta} \gamma_2(S, d)
  \]
  Optimizing for the right-hand side being less than \(\epsilon\) gives, if \(\epsilon \ll \frac{(2L\gamma_2(S, d))^2}{C}\), \(\delta^{\star}(\epsilon) = \frac{\epsilon^2}{4L^2 \gamma_2^2(S,d)}\)
\end{claim}
\noindent The constant \(C\) here is the consistency error from the Bellman operator (see Appendix~\ref{app:barles}) and as such, the hypothesis that this is bounded holds if Claim~\ref{claim:barles} holds.
This claim implies that the fineness that the grid needs to be in order to have a bounded error is dominated by \(\gamma_2^2(S,d)\).
In particular, this effectively states that our Markov operator needs to be discretized to \(\Omega(\frac{1}{\gamma_2^2(S, d)})\) in order for the discrete and continuous operators to match up to \(\epsilon\).

\paragraph{Bellman Operator.}
We show a similar result for the Bellman operator.
\begin{claim}\label{claim:operator-error}
  Let \(u: S\to\reals\) be \(L_{u}\)–Lipschitz. 
  For each mesh‐size \(\delta>0\), let \(S_{\delta}\subset S\) be a \(\delta\)–net.
  Define:
  \[
  \mathcal{T}^{\delta}[u](s_{i}) = \max_{a\in A}\left\{r(s_{i},a) +\tfrac1\delta\sum_{s_{j}\in S_{\delta}} P^{\delta}(s_{j}\!\mid s_{i},a)\left[u(s_{j})-u(s_{i})\right]\right\}
  \]
  extending this to all \(s\in S\) by choosing any nearest grid‐point \(s_{i}\in S_{\delta}\), and the continuous operator
  \[
  \overline{\mathcal{T}}\,u(s) = \beta\,u(s) -\sup_{a\in A}\left\{\,r(s,a)+\mathcal{L}^{a}u(s)\right\} \quad
  \mathcal{L}^{a}u = b(s,a)\cdot\nabla u(s) + \tfrac12\Tr\left[\sigma\sigma^{T}(s,a)\,\nabla^{2}u(s)\right]
  \]
  Assume the pointwise consistency bound
  \(\sup_{s_{i}\in S_{\delta}}|\mathcal{T}^{\delta}[u](s_{i})-\overline{\mathcal{T}}u(s_{i})|\le C\,\delta\).
  Then
  \[
  \sup_{s\in S}\left|\mathcal{T}^{\delta}[u](s)-\overline{\mathcal{T}}u(s)\right|
  \leq C\,\delta + 2L_{u}\gamma_{2}(S,d)\,\sqrt{\delta}
  \]
\end{claim}

\section{Formal Results}\label{sec:results}
We establish bounds for forward and backward Bellman iterations, showing that forward iteration hitting times have a $\Theta(\gamma_2^2(S,d))$ gap between upper and lower bounds.
For backward iteration, we prove that while error rates remain similar, the backward hitting time matches the forward iteration's lower bound asymptotically.
This yields a key insight: the backward-to-forward hitting time ratio is bounded by $O\left(\frac{\gamma_2(\mathrm{supp}(r), d)}{\gamma_2(S,d)}\right)$, showing that sparse rewards enable significantly faster convergence through backward iteration.

Our analysis of dynamic regret demonstrates that iterations achieve $\tilde{O}(\sqrt{T})$ regret without gap conditions, while $(\Delta,\epsilon)$-gapped rewards enable $O(\log T)$ regret.
These results form the foundation for our BARS algorithm (\S\ref{sec:BARS}), which dynamically scales sparse rewards to optimize convergence with bounded regret.

Figure~\ref{fig:iteration-comparison} illustrates how stronger signals from sparse rewards give backward search a higher success probability, and shows Bellman iteration approximating a viscosity solution of a limiting HJB equation.

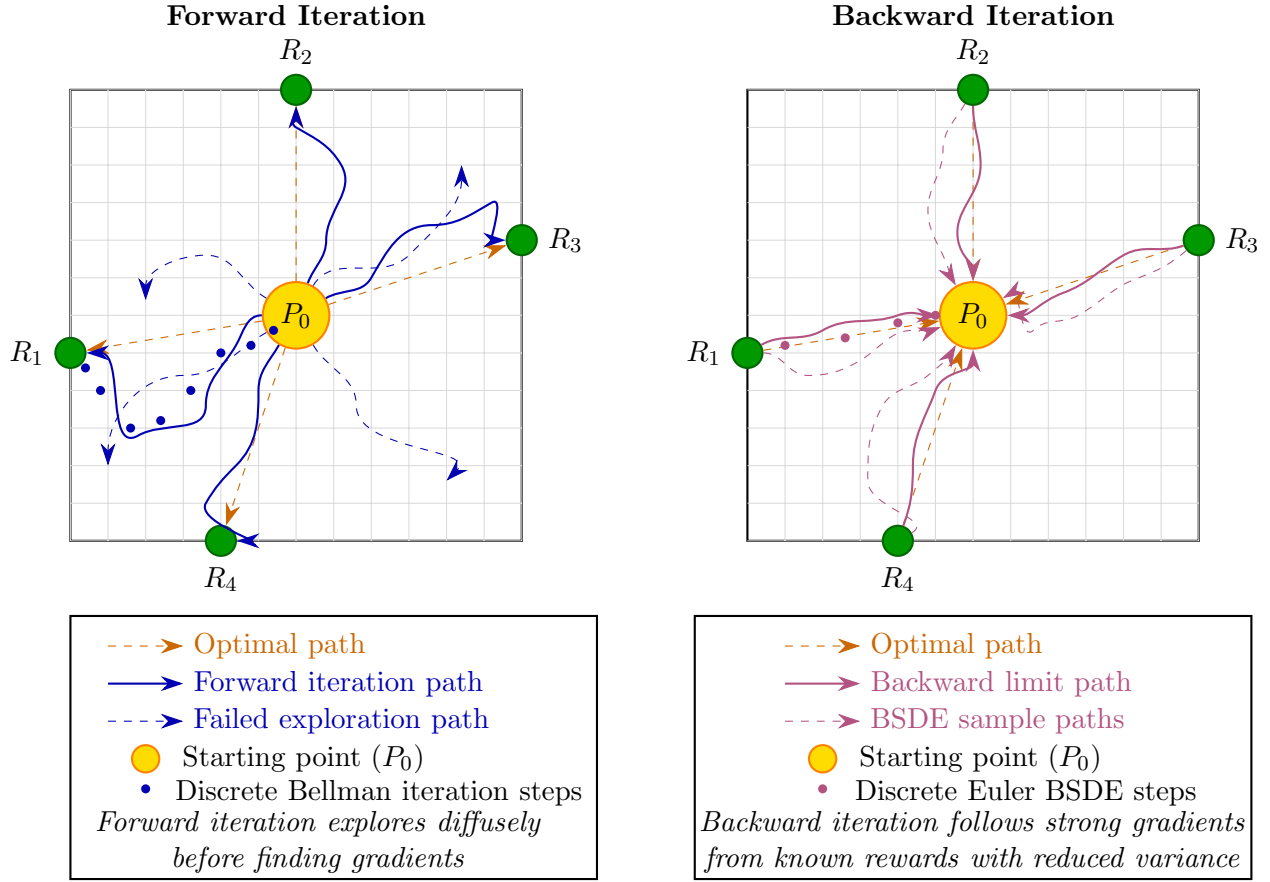
\begin{figure}
  \centering
    \begin{tikzpicture}[
      forward/.style={thick, blue!70!black, -{Stealth[length=3mm, width=2mm]}},
      forward-failed/.style={dashed, blue!70!black, -{Stealth[length=3mm, width=2mm]}},
      backward/.style={thick, magenta!70!black, -{Stealth[length=3mm, width=2mm]}},
      backward-var/.style={dashed, magenta!70!black, -{Stealth[length=3mm, width=2mm]}},
      optimal/.style={dashed, orange!80!black, -{Stealth[length=3mm, width=2mm]}},
      grid/.style={step=0.5, gray!30, very thin},
      subtitles/.style={font=\normalsize\bfseries}
  ]
  
  \draw[thick] (0, 0) rectangle (6, 6);
  \draw[grid] (0, 0) grid (6, 6);
  
  \node[subtitles] at (3, 7) {Forward Iteration};
  
  \node[circle, fill=green!60!black, draw=green!40!black, thick, minimum size=0.4cm, label=left:$R_1$] (R1-f) at (0, 2.5) {};
  \node[circle, fill=green!60!black, draw=green!40!black, thick, minimum size=0.4cm, label=above:$R_2$] (R2-f) at (3, 6) {};
  \node[circle, fill=green!60!black, draw=green!40!black, thick, minimum size=0.4cm, label=right:$R_3$] (R3-f) at (6, 4) {};
  \node[circle, fill=green!60!black, draw=green!40!black, thick, minimum size=0.4cm, label=below:$R_4$] (R4-f) at (2, 0) {};
  
  \node[circle, fill=yellow!80!orange, draw=orange, thick, minimum size=0.5cm] (P0-f) at (3, 3) {$P_0$};
  
  \draw[optimal] (P0-f) -- (R1-f);
  \draw[optimal] (P0-f) -- (R2-f);
  \draw[optimal] (P0-f) -- (R3-f);
  \draw[optimal] (P0-f) -- (R4-f);
  
  \draw[forward] (P0-f) to[out=180, in=60] (2.2, 2.6) to[out=240, in=90] (1.8, 1.8) 
      to[out=270, in=30] (0.9, 1.4) to[out=210, in=0] (R1-f);
  
  \draw[forward] (P0-f) to[out=70, in=270] (3.4, 4.0) to[out=90, in=300] (3.6, 5.0) 
      to[out=120, in=330] (3.0, 5.5) to[out=150, in=270] (R2-f);
  
  \draw[forward] (P0-f) to[out=30, in=240] (4.0, 3.5) to[out=60, in=180] (4.8, 4.2) 
      to[out=0, in=210] (5.6, 4.5) to[out=30, in=180] (R3-f);
  
  \draw[forward] (P0-f) to[out=240, in=90] (2.5, 2.0) to[out=270, in=45] (2.2, 1.0) 
      to[out=225, in=60] (1.8, 0.5) to[out=240, in=0] (R4-f);
  
  \draw[forward-failed] (P0-f) to[out=150, in=0] (1.8, 3.8) to[out=180, in=90] (1.0, 3.2);
  \draw[forward-failed] (P0-f) to[out=60, in=225] (4.5, 3.8) to[out=45, in=270] (5.2, 5.0);
  \draw[forward-failed] (P0-f) to[out=300, in=90] (3.8, 1.8) to[out=270, in=45] (5.0, 0.8);
  \draw[forward-failed] (P0-f) to[out=210, in=0] (1.5, 2.2) to[out=180, in=90] (0.5, 1.0);
  
  \foreach \i/\x/\y in {1/2.7/2.8, 2/2.4/2.6, 3/2.0/2.5, 4/1.6/2.0, 5/1.2/1.6, 6/0.8/1.5, 7/0.4/2.0, 8/0.2/2.3} {
      \fill[blue!70!black] (\x,\y) circle (0.06);
  }
  
  \draw[thick] (9, 0) rectangle (15, 6);
  \draw[grid] (9, 0) grid (15, 6);
  
  \node[subtitles] at (12, 7) {Backward Iteration};
  
  \node[circle, fill=green!60!black, draw=green!40!black, thick, minimum size=0.4cm, label=left:$R_1$] (R1-b) at (9, 2.5) {};
  \node[circle, fill=green!60!black, draw=green!40!black, thick, minimum size=0.4cm, label=above:$R_2$] (R2-b) at (12, 6) {};
  \node[circle, fill=green!60!black, draw=green!40!black, thick, minimum size=0.4cm, label=right:$R_3$] (R3-b) at (15, 4) {};
  \node[circle, fill=green!60!black, draw=green!40!black, thick, minimum size=0.4cm, label=below:$R_4$] (R4-b) at (11, 0) {};
  
  \node[circle, fill=yellow!80!orange, draw=orange, thick, minimum size=0.5cm] (P0-b) at (12, 3) {$P_0$};
  
  \draw[optimal] (R1-b) -- (P0-b);
  \draw[optimal] (R2-b) -- (P0-b);
  \draw[optimal] (R3-b) -- (P0-b);
  \draw[optimal] (R4-b) -- (P0-b);
  
  \draw[backward] (R1-b) to[out=30, in=210] (10.0, 2.8) to[out=30, in=200] (11.0, 3.0) to[out=20, in=180] (P0-b);
  
  \draw[backward] (R2-b) to[out=270, in=90] (12.1, 5.0) to[out=270, in=100] (11.8, 4.0) to[out=280, in=90] (P0-b);
  
  \draw[backward] (R3-b) to[out=200, in=30] (14.0, 3.8) to[out=210, in=20] (13.2, 3.3) to[out=200, in=0] (P0-b);
  
  \draw[backward] (R4-b) to[out=70, in=270] (11.2, 1.0) to[out=90, in=250] (11.5, 2.0) to[out=70, in=270] (P0-b);
  
  \draw[backward-var] (R1-b) to[out=0, in=180] (9.8, 2.2) to[out=0, in=210] (10.5, 2.4)
      to[out=30, in=180] (11.2, 2.8) to[out=0, in=200] (P0-b);
      
  \draw[backward-var] (R2-b) to[out=240, in=90] (11.6, 5.2) to[out=270, in=90] (11.4, 4.4)
      to[out=270, in=100] (11.7, 3.5) to[out=280, in=120] (P0-b);
      
  \draw[backward-var] (R3-b) to[out=220, in=20] (14.2, 3.5) to[out=200, in=30] (13.4, 3.0)
      to[out=210, in=45] (12.8, 2.8) to[out=225, in=30] (P0-b);
      
  \draw[backward-var] (R4-b) to[out=30, in=270] (10.6, 0.8) to[out=90, in=240] (10.8, 1.8)
      to[out=60, in=220] (11.4, 2.2) to[out=40, in=240] (P0-b);
  
  \foreach \i/\x/\y in {1/9.5/2.6, 2/10.3/2.7, 3/11.0/2.9, 4/11.5/3.0} {
      \fill[magenta!70!black] (\x,\y) circle (0.06);
  }
  
  \draw[thick] (0, -1) rectangle (7, -4.5);
  \draw[optimal] (0.5, -1.4) -- (1.5, -1.4) node[right] {Optimal path};
  \draw[forward] (0.5, -1.9) -- (1.5, -1.9) node[right] {Forward iteration path};
  \draw[forward-failed] (0.5, -2.4) -- (1.5, -2.4) node[right] {Failed exploration path};
  \node[circle, fill=yellow!80!orange, draw=orange, thick, minimum size=0.3cm] at (1.0, -2.9) {};
  \node at (3.1, -2.9) {Starting point ($P_0$)};
  \fill[blue!70!black] (1, -3.3) circle (0.06);
  \node at (4.1, -3.35) {Discrete Bellman iteration steps};
  
  \draw[thick] (8.3, -1) rectangle (15.7, -4.5);
  \draw[optimal] (9.5, -1.4) -- (10.5, -1.4) node[right] {Optimal path};
  \draw[backward] (9.5, -1.9) -- (10.5, -1.9) node[right] {Backward limit path};
  \draw[backward-var] (9.5, -2.4) -- (10.5, -2.4) node[right] {BSDE sample paths};
  \node[circle, fill=yellow!80!orange, draw=orange, thick, minimum size=0.2cm] at (10.0, -2.9) {};
  \node at (12.1, -2.9) {Starting point ($P_0$)};
  \fill[magenta!70!black] (10, -3.3) circle (0.06);
  \node at (12.7, -3.35) {{Discrete Euler BSDE steps}};
  
  \node[align=center] at (3.3, -4.0) {\textit{Forward iteration explores diffusely}\\\textit{before finding gradients}};
  \node[align=center] at (12, -4.0) {\textit{Backward iteration follows strong gradients}\\\textit{from known rewards with reduced variance}};
  
  \end{tikzpicture}
  \caption{Comparison of forward and backward iteration approaches for sparse rewards. 
  Forward iteration (left) requires extensive exploration with many failed paths before finding optimal paths to rewards. 
  Backward iteration (right) propagates value directly from rewards with stronger gradients and less variance near the goal point, resulting in more efficient computation. 
  The blue dots in the forward panel represent discrete steps of forward Bellman value iteration, while the magenta dots in the backward panel represent steps of backward Euler BSDE iteration. 
  The continuous paths represent their scaling limits as $\delta \to 0$, illustrating the key finding that backward iteration requires $\tau^-_\epsilon = \Theta(\gamma_2(\mathrm{supp}(r),d)^2/\epsilon^2)$ steps compared to forward iteration's $\tau^+_\epsilon = \Theta(\gamma_2(S,d)^2/\epsilon^2)$ steps.}
  \label{fig:iteration-comparison}
\end{figure}

\subsection{Forward Hitting Times}\label{sec:hitting-times}
\begin{claim}[Low Coupling Error between MDP and Controlled Diffusion]\label{claim:coupling-error}
  Fix a control sequence of actions \(\{a_{k}\}_{k \in \mathbf{N}} \subset A\). 
  Consider a continuous and discrete evolution that involves a time step \(\delta>0\).
  Define the following sample paths:
  \begin{itemize}
    \item \(S_{t}\) solve the controlled SDE
      \[
        dS_{t} = b\left(S_{t},a_{\lfloor t/\delta\rfloor}\right)\,dt + \sigma\left(S_{t},a_{\lfloor t/\delta\rfloor}\right)\,dW_{t}
        \quad S_{0}=x_{0}
      \]
    \item \(\tilde S_{k\delta}\) be its Euler–Maruyama discretization on the same Brownian path:
      \[
        \tilde S_{(k+1)\delta} = \tilde S_{k\delta} + b\left(\tilde S_{k\delta},a_{k}\right)\,\delta + \sigma\left(\tilde S_{k\delta},a_{k}\right)\,\left(W_{(k+1)\delta}-W_{k\delta}\right)
        \quad \tilde S_{0}=x_{0}
      \]
    \item \(X_{k}\) be the controlled Markov chain with transitions
      \(\displaystyle X_{k+1}\sim P^{\delta}(\cdot\mid X_{k},a_{k})\)
      coupled to the same Brownian increments via a sub‑Gaussian coupling.
  \end{itemize}
  Assume:
  \begin{enumerate}
    \item \(b,\sigma\) are \(L\)‑Lipschitz in state, uniformly in action.
    \item Local consistency:
      \(\Expect[\Delta X\mid s,a]=b(s,a)\,\delta+O(\delta^{2})\),  
      \(\Var[\Delta X\mid s,a]=\sigma\sigma^{T}(s,a)\,\delta+O(\delta^{2})\).
    \item Sub‑Gaussian jumps: \(\exists\,c>0\) such that \(\Expect\left[e^{\lambda\cdot(\Delta X - \Expect[\Delta X])}\mid s,a\right] \le \exp\left(c\,\|\lambda\|^{2}\,\delta\right),\; \forall\,\lambda\in\reals^{d}\)
  \end{enumerate}
  Then \(\exists K>0\) (depending on \(L,c\)) so that for any \(r>0\) and integer \(N\),
  \[
    \Prob\!\left(\max_{0\le k\le N}\|X_{k}-\tilde S_{k\delta}\|\ge r\right)
    \;\le\;
    2\exp\!\left(-\frac{r^{2}}{K\,\delta\,N}\right).
  \]
  \end{claim}
\noindent We prove this in Appendix~\ref{app:coupling-error}.
This result demonstrates, by comparison to a discretization of the limiting SDE, that the error rate of how fast the MDP approaches its limit can be controlled.
From this tail-bound, we are able to construct an explicit proof of a bound on the hitting time \(\tau^{+}_{\epsilon}\):
\begin{claim}[Hitting Time of Forward Iteration]\label{claim:hitting-time}
  Under the hypotheses of Claim~\ref{claim:coupling-error}, let \(S_t\) solve the controlled SDE and \(X_k\) its discretized Markov chain coupled via a sub‐Gaussian coupling constant \(K\).
  Then for any \(\epsilon>0\), \(\delta>0\), and \(p\in(0,1)\), with probability at least \(1-p\),
  \[
  \frac{\gamma_2(S,d)^{2}}{\epsilon^{2}}\;\ln\!\frac2p
  \;<\;\tau^+_{\epsilon}\;<\;
  \frac{L^2 \gamma_2(S, d)^2}{K}\;\ln\!\frac4p.
  \]
  implying that $\tau_{\epsilon}^+ = \Theta(\gamma_2(S, d)^2)$.
\end{claim}
\noindent We prove this in Appendix~\ref{app:forward-hitting-time}.
This shows that the time to reach the accuracy $\epsilon$ relative to the continuous limit is tightly controlled by $\gamma_2(S, d)$.

\subsection{Backwards Hitting Time}\label{sec:backward-hitting-time}
\begin{claim}[Hitting Time of Backward Iteration]\label{claim:backward-hitting-time}
  Assume that controlled diffusion has drift \(b(s,a)\) and diffusion \(\sigma(s,a)\) that are uniformly Lipschitz (with constant \(L>0\)), and that the discrete backward Euler BSDE operator \(T_\delta\) is monotone, non--expansive, and consistent in the sense that for any smooth function \(u\) $\|T_\delta[u]-T[u]\|_\infty \le C\,\delta$,
  Then the terminal error satisfies $\|Y^\delta_N - Y_N\|_\infty \le 2L\,\sqrt{\delta}\,\gamma_2\left(\mathrm{supp}(r),d\right)$.
  If one selects the mesh size $\delta^* = \frac{\epsilon^2}{16L^2\,\gamma_2\left(\mathrm{supp}(r),d\right)^2}$, the discrete backward iteration error \(\|Y^\delta_0-Y_0\|_\infty\) is at most \(\epsilon\) provided that the number of backward iterations \(\tau^-_\epsilon\) (i.e. the number of steps required so that the cumulative error is below \(\epsilon\)) satisfies
  \[
  \frac{8L^2}{C_1}\,\frac{\gamma_2\left(\mathrm{supp}(r),d\right)^2}{\epsilon} \le \tau^-_\epsilon \le \frac{8L^2}{C_2}\,\frac{\gamma_2\left(\mathrm{supp}(r),d\right)^2}{\epsilon},
  \]
  for some constants \(C_1,C_2>0\). In particular, $\tau^-_\epsilon = \Theta\!\left(\frac{L^2\,\gamma_2\left(\mathrm{supp}(r),d\right)^2}{\epsilon}\right)$.
\end{claim}
\noindent We prove this in Appendix~\ref{app:backward-hitting-time}.
The most important thing to note here is that the metric complexity of the support of the rewards completely determines the backward hitting time.
If the rewards are sparse, one is likely to have $\gamma_2(\mathrm{supp}(r), d) = o(\gamma_2(S,d))$. 

\subsection{Bound on Ratios of Hitting‐Times}
We next bound the ratio of the backward to forward hitting times.
We show that this ratio is bounded by the ratio of the $\gamma_2$ evaluated on $\mathrm{supp}(r)$ and $S$. 
This implies that if $\mathrm{supp}(r)$ is sufficiently sparse relative to $S$, then 

\begin{claim}[Ratio of Hitting Times]\label{claim:hitting-time-ratio}
Under the hypotheses of Claims 3.1 and 4.3, fix $\epsilon>0$ and $p\in(0,1)$.  Define
\begin{align*}
\delta \;=\;\min\!\left\{\tfrac{\epsilon}{2C},\;\left(\tfrac{\epsilon}{4L\,\gamma_2(S,d)}\right)^2\right\}
&&
N^+ \;=\;\frac{\gamma_2(S,d)^2}{\epsilon^2}\,\ln\!\frac{2}{p}
&&
N^- \;=\;\frac{\gamma_2\left(\mathrm{supp}(r),d\right)^2}{\epsilon^2}\,\ln\!\frac{2}{p}
\end{align*}
Let
\[
\tau^+ \;=\;\min\{\,k:\|V^\delta_k - V\|_\infty\le\epsilon\}
\qquad
\tau^- \;=\;\min\{\,k:\|Y^\delta_k - Y\|_\infty\le\epsilon\}
\]
be the forward and backward hitting times.  Then with probability at least $1-2p$,
\[
\tau^- \;\le\; N^-, 
\quad
\tau^+ \;>\; N^+
\qquad\Longrightarrow\qquad
\frac{\tau^-}{\tau^+}\;\le\;\left(\tfrac{\gamma_2(\mathrm{supp}(r),d)}{\gamma_2(S,d)}\right)^2
\]
\end{claim}
This bound is what justifies the usage of sparse rewards: If one has sparse rewards, \ie~$\gamma_2(\mathrm{r},d) = o(\gamma_2(S, d))$, then the backward hitting time is asymptotically faster.
One can also interpret this as showing that the only thing controlling relative hitting rates is the metric entropy of the reward support and $S$.

\subsection{Minimum and Maximum Rewards}
Reasoning models aim to provide concise reasoning chains, which depend on the magnitude and support of the rewards.
If rewards are too small compared to noise, exploration becomes pure diffusion without finding the target state.
Conversely, if rewards are too large, $\epsilon$-optimal rewards may be achieved without reaching the target.
We present reward bounds to avoid these issues in CoT reasoning.

\paragraph{Minimum Rewards.}
When one discretizes a controlled diffusion, there are natural limits on how small the reward can be for the error between discretization and the continuous limit to be small.
In particular, if the reward $r(s, a)$ generates a signal much smaller than the variance of the diffusion, then the process can become unstable.
We formalize this by modifying Claim~\ref{claim:coupling-error} to explicitly provide bounds on how large the rewards need to be.

\begin{claim}[Minimum Reward for Accurate Discretization]\label{claim:unified_eps_accuracy}
  Let $\{Y_k^\delta\}_{k\ge0}$ be the backward Euler iterates with mesh size $\delta>0$ and a sparse reward of constant magnitude $r$.
  Assume that the controlled Markov chain satisfies the sub‑Gaussian coupling of Claim 4.1 (with constant $K$) and that drift and diffusion are $L$‑Lipschitz.  Fix an error tolerance $\epsilon>0$ and confidence level $p\in(0,1)$.
  Choose
  \begin{align*}
  \delta = \left(\frac{\epsilon}{4\,L\,\gamma_2(\mathrm{supp}(r), d)}\right)^{2}
  &&
  N = \frac1\delta = \left(\frac{4\,L\,\gamma_2(\mathrm{supp}(r), d)}{\epsilon}\right)^{2}.
  \end{align*}
  Then
  \[
  \Prob\left(\|Y_N^\delta - Y_N\|_\infty \le \epsilon\right)\ge 1-p
  \quad\Longleftrightarrow\quad
  r\ge\sqrt{K\ln\frac{2}{p}\,}.
  \]
\end{claim}
\noindent We prove this claim in Appendix~\ref{app:unified_eps_accuracy}.
This claim suggests that any algorithm to learn the optimal rewards to minimize the number of reasoning steps while ensuring low regret needs the minimum reward to be sufficiently large.
  
There are also geometric constraints on $\mathrm{supp}(r)$.
If the total reward over the support is too small, then it is possible that no trajectories $\tau$ can reach the target during backward evolution.
We formalize this using~\emph{effective reward mass}, which is the integral of the rewards on their support.
In particular, suppose that we have a probability measure $\mu$ on $(\mathrm{supp}(r), d)$; then we define the effective reward mass as
\[
I_r = \int_{\mathrm{supp}(r)} r(s) d\mu(s)
\]
We show that if the effective reward mass is smaller than $\gamma_2(\mathrm{supp}(r), d)$ times the noise level of Claim~\ref{claim:unified_eps_accuracy}, then realizing $\epsilon$ error is difficult.
\begin{claim}[Effective Reward Mass Lower Bound]\label{claim:effective-reward-mass-lower-bound}
  Let $r$ be bounded (and hence, Lipschitz) and assume the hypotheses of Claim~\ref{claim:unified_eps_accuracy}.
  Then a sufficient condition for $\Prob\left(\|Y_N^\delta - Y_N\|_\infty \le \epsilon\right)\ge 1-p$ is
  \[
  I_r \ge \sqrt{K \ln (2/p)} + C_0 L \gamma_2(T,d),
  \]
  where $C_0>0$ is an absolute constant.
\end{claim}
\noindent We prove this in Appendix~\ref{app:effective-reward-mass-lower-bound}.

\paragraph{Maximum Rewards.}
On the other hand, to prevent reward hacking~\cite{skalse2022defining,shao2024deepseekmath,guo2025deepseek}, we need to ensure that the minimum reward is not too large relative to the value at the target state $s_0$.
Reward hacking often occurs when the agent manipulates the reward function to achieve an $\epsilon$-optimal value by overexploring and not reaching the target state.
DeepSeek~\cite{guo2025deepseek} observed that reward hacking occurred when process supervision was used in their previous experiments (\eg~\cite{shao2024deepseekmath}), which is why they used sparse outcome-based rewards.

A natural exploit of the backward exploration when rewards are too dense is to have the search process loop around a region of high rewards without ever reaching the target user query.
We provide a simple upper bound for rewards that minimize such looping.
In practice, one needs to dynamically adjust the upper bound, using dynamic algorithms such as the one we present in~\ref{sec:BARS}.
\begin{claim}[Loop Exploit Upper Bound]\label{claim:loop_exploit_bound}
  Let $(S,A,P,\gamma)$ be an MDP with discount factor $\gamma\in[0,1)$, a distinguished ``target'' state $s_0\in S$ with optimal value $V^*(s_0)$, and a set of ``sparse'' states $S_r\subset S\setminus\{s_0\}$ with constant reward $r>0$ (and zero elsewhere).
  The ``loop'' policy that remains in any $s\in S_r$ forever achieves value $V_{\rm loop}=\frac{r}{1-\gamma}$.
  To ensure that reaching $s_0$ is strictly better than looping (avoiding reward hacking), we must have $r<(1-\gamma)(V^*(s_0)-\epsilon)$ for any desired optimality gap $\epsilon\geq 0$.
\end{claim}
\noindent We prove this in Appendix~\ref{app:loop_exploit_bound}.

\subsection{Regret Bounds}

We measure performance against the best fixed policy in hindsight using the following definition.

\begin{definition}[Static Regret]
Let $\{\pi_k\}_{k=0}^{T-1}$ be the sequence of greedy policies induced by approximate value functions $\{V_k\}$.
Denote by $J(\pi)=\Expect_{s\sim d^{\pi}}\left[V^{\pi}(s)\right]$ the expected return of policy $\pi$, and let
$\pi^*\in\arg\max_{\pi}J(\pi)$ be an optimal fixed policy.
Then the \emph{static regret} over $T$ steps is
\[
R_T=\sum_{k=0}^{T-1}\left(J(\pi^*)-J(\pi_k)\right)
\]
\end{definition}
\noindent We relate $R_T$ to $\tau^-_\epsilon$ and write $E_0=\|V_0-V^*\|_\infty$.
\begin{claim}[Regret Bounds]\label{claim:regret-hitting}
Suppose that the rewards are nonnegative and bounded by $R_{\max}$, and that with probability at least $1-p$ the hitting time satisfies
\[
\tau^-_\epsilon\le N(\epsilon)=C_1\frac{(R_{\max}\gamma_*)^2}{\epsilon^2}\log\frac1p
\]
for some constant $C_1$ where $\gamma_{\star} = \min(\gamma_2(S, d), \gamma_2(\mathrm{supp}(r), d))$.
Then with probability at least $1-p$,
\[
R_T
= \sum_{k=0}^{T-1}\left(J(\pi^*)-J(\pi_k)\right)
\le 2E_0\min\{T,\tau^-_\epsilon\} + 2\epsilon\max\{0, T-\tau^-_\epsilon\}
\le 2E_0\,N(\epsilon) + 2\epsilon\,T
\]
Moreover, choosing $\epsilon = R_{\max}\gamma_*\sqrt{\frac{\log(1/p)}{T}}$ balances the two terms and yields
\[
R_T = O\left(R_{\max}\gamma_*\sqrt{T\log\tfrac1p}\right)
\]
\end{claim}
\noindent We prove this in Appendix~\ref{app:regret-hitting}.

\begin{claim}[Logarithmic Regret under a $(\Delta,\epsilon)$–Gap]\label{claim:regret-gap}
  Assume that the sparse reward $r(s,a)\in[0,R_{\max}]$ satisfies the $(\Delta,\epsilon)$–gap condition~\ref{eq:gap-condition}.
  Let $V_k$ be the value iteration sequence $V_{k+1}=\mathcal{T}^{\delta}_{\mathrm{BDSE}}[V_k]$ on the discrete mesh with initial conditions at the sparse rewards and define $\delta_k \;=\;\|V_k - V^*\|_\infty$. 
  Then the backward hitting time $\tau^-_{\epsilon}$ obeys $\tau^-_\epsilon \;=\; O\!\left(\tfrac{R_{\max}}{\Delta}\,\ln\tfrac{R_{\max}}{\epsilon}\right)$, and hence by the performance difference lemma, $R_T = O\!\left(\tfrac{R_{\max}}{\Delta}\,\ln T\right)$.
\end{claim}
\noindent We prove this in Appendix~\ref{app:regret-gap}.

\section{Backwards Adaptive Reward Shaping (BARS)}\label{sec:BARS}
We now integrate the lower and upper bounds on the sparse reward scale derived in Sections~3 and 4 into an online algorithm to choose the optimal rewards.
By enforcing these bounds on the scaling parameter per round $\lambda_t$, the algorithm simultaneously achieves the following:
\begin{enumerate}
  \item \emph{Low dynamic regret}, namely $R_T^{\mathrm{dyn}}=O(\log T)$
  \item \emph{Fast convergence per round}, namely hitting $\epsilon$-accuracy in only $O(1/\epsilon^2)$ backward steps
\end{enumerate}
Unlike the static regret bounds of the previous section, here we optimize for dynamic regret where the reward function changes in each round.

\paragraph{Dynamic Regret and Hitting Time.}
We start with an MDP $(S, A, P, \tilde{r}, \gamma)$, where $\tilde{r}$ is the \emph{base rewards}.
The goal of BARS is to modify the based rewards multiplicatively based on received samples. 
In each round $t$, BARS draws a sample $s_t$ from a known prior distribution $p(s)$ over the set of correct terminal states.
We assume oracle access to the corresponding optimal action $a^*(s_t)$ that leads from $s_t$ toward the true goal $s_0$.  
The algorithm then assigns a positive, sparse reward
\[
  r_t\left(s_t,a^*(s_t)\right)=\lambda_t\,\tilde r\left(s_t\right)
\]
while all other state–action pairs receive zero reward.  
The support set $S_t$ is updated to include each sampled $s_t$, and the backward Euler solver uses this growing set of reward-bearing transitions to propagate the value back to the initial query state.
We note that our algorithm can be viewed as a variant of the empirical regret minimization algorithm of~\cite{qu2025optimizing}, albeit with theoretical guarantees.

\begin{itemize}
  \item $J_t(\pi)=\Expect_\pi\left[\sum_{k=0}^\infty\gamma^k r_t(s_k,a_k)\mid s_0\right]$, the expected return under policy $\pi$ at round $t$.
  \item $J_t^*=\max_\pi J_t(\pi)$, the optimal return for reward $r_t$.
  \item The \emph{dynamic regret} over $T$ rounds:
  \[ 
     R_T^{\mathrm{dyn}} = \sum_{t=1}^T\left(J_t^* - J_t(\pi_t)\right)
  \]
  \item Within round $t$, we run a backward Euler (BSDE) solver.  Let $V_k^{\delta,t}$ be the $k$th iterate on grid mesh $\delta$, and let $V^*$ denote the true value function under $r_t$.  The \emph{hitting time} to $\epsilon$-accuracy is
  \[ 
    \tau_{\epsilon}^{(t)} = \min\{k:\|V_k^{\delta,t}-V^*\|_\infty\le\epsilon\}
  \]
\end{itemize}

\paragraph{Reward Constraints.}
Let $\tilde r_{\min}>0$ and $\tilde r_{\max}>0$ be the minimum and maximum of the baseline reward.
We enforce two constraints:
\begin{enumerate}
\item \emph{Lower bound (signal vs.\ noise)}: By Claim~\ref{claim:unified_eps_accuracy}, to overcome BSDE variance with confidence $1-p$, we require $\lambda_t \geq \lambda_{\min} := \frac{\sqrt{2c\log(2/p)}}{\tilde r_{\min}}$
\item \emph{Upper bound (no looping exploit)}: By Claim~\ref{claim:loop_exploit_bound}, to prevent a ``loop'' policy from achieving $J_t^*$, we require $\lambda_t < \lambda_{\max} := \frac{(1-\gamma)J_t^*}{\tilde r_{\max}}$
\end{enumerate} 
Hence we always maintain $\lambda_{\min} \le \lambda_t < \lambda_{\max}$.

\subsection{BARS Algorithm with Bounds}
\begin{algorithm}[H]
\caption{BARS with Reward-Scale Bounds}
\label{alg:BARS_bounds}
\begin{algorithmic}[1]
  \REQUIRE Baseline reward $\tilde r$, gap $\Delta_0$, discount $\gamma$, confidence $p$, Lipschitz $L_0$, complexity parameter $\alpha$, bounds $(\lambda_{\min},\lambda_{\max})$, grid $\delta$, initial $V^0$
  \FOR{$t=1$ to $T$}
    \STATE Sample $x_t\sim p(s)$; update support $S_t\leftarrow S_{t-1}\cup\{x_t\}$
    \STATE Estimate $\hat\gamma_t\approx\gamma_2(\mathrm{supp}\,p)$ from $S_t$ (via, \eg Cover Trees~\cite{Beygelzimer2006}; see Appendix~\ref{app:eps-net})
    \STATE Set $\lambda_t=\mathrm{clip}(\alpha/\hat\gamma_t,[\lambda_{\min},\lambda_{\max}])$
    \STATE Define $r_t(s)=\lambda_t\,\tilde r(s)$; run backward Euler on grid $\delta$ to compute $V^t$
    \STATE At state $s_t$, pick
      \[
        a_t = \arg\max_{a\in A}\{r_t(s_t,a)+\gamma\Expect[V^t(s')\mid s_t,a]\}
      \]
    \STATE Observe $s_{t+1}$, incur regret $J_t^*-J_t(\pi_t)$
  \ENDFOR
\end{algorithmic}
\end{algorithm}
\noindent We note that the algorithm uses an estimate of $\hat{\gamma}_t$ to clip the reward scale.
While estimating $\hat{\gamma}_t$ is computationally expensive, we show that it feasible via $\epsilon$-net arguments in Appendix~\ref{app:eps-net}.
In particular, as the user is choosing the sparsity of the reward via the prior distribution $p(s, a)$, they can sparsify it (\eg~via Johnson-Lindenstrauss) prior to performing $\hat{\gamma}_t$ estimates. 

\paragraph{Proof of Correctness.}
We claim Algorithm~\ref{alg:BARS_bounds} achieves low dynamic regret efficiently.
\begin{claim}\label{claim:BARS_bounds}
  Under the MDP and regularity assumptions of Sections 2–4, suppose at each round \(t\) we choose a per‐round reward \(r_t(s,a)=\lambda_t\,\tilde r(s)\) with scaling \(\lambda_{\min}\le\lambda_t<\lambda_{\max}\), and assume \(r_t\) is \((\Delta_0,\epsilon_t)\)‑gapped with \(\epsilon_t=1/t\).
  Let \(\{V^{\delta}_t\}_{t\ge0}\) be the backward Euler BSDE iterates on mesh size \(\delta\), initialized at \(V^{\delta}_0\), and \(\overline{V}^*_t\) be the true value function for \(r_t\). Then with high probability:
  \begin{enumerate}[label=(\roman*)]
    \item \emph{Polynomial hitting time.} The hitting time \(\tau_t = \min\{k:\|V^{\delta,t}_k - V^*_t\|_\infty \le \epsilon_t\}\) satisfies \(\tau_t = O(\gamma_2(\mathrm{supp}(r_t),d)^2/\epsilon_t^2)\)
    \item \emph{Logarithmic dynamic regret.} If \(\pi_t\) is the greedy policy from \(V^{\delta,t}_{\tau_t}\), then \(R^{\mathrm{dyn}}_T = \sum_{t=1}^T(J^*_t - J_t(\pi_t)) = O(\sum_{t=1}^T \epsilon_t) = O(\log T)\)
  \end{enumerate}
  \end{claim}
\noindent We prove this in Appendix~\ref{app:BARS_bounds}.
One can interpret this result as effectively justifying the usage of very long CoTs (such as those found in DeepSeek R1, Gemini 2.5 Pro, or OpenAI's O1 and O3) as the regret does not accumulate significantly.

\section{Conclusion and Future Work}

The transformation of sparse, outcome-only rewards into dense, procedure-based signals represents a fundamental challenge in chain-of-thought reasoning systems.
In this work, we introduced Backward Adaptive Reward Shaping (BARS), a theoretically grounded framework that dynamically scales per-step rewards through black-box access to terminal-state priors and online estimates of the Talagrand $\gamma_2$ functional, while enforcing critical signal-to-noise and anti-looping constraints.
These results can be viewed as providing theoretical justification for sparse outcome reward shaping in very long chains-of-thought which is key to the empirical success of DeepSeek R1~\cite{guo2025deepseek} and Gemini 2.5 Pro~\cite{setlur2024rewarding}.

Our analysis establishes three key theoretical contributions: (i) \emph{accelerated convergence}, where our backward Euler solver achieves $\epsilon$-accuracy in $O(\gamma_2(\mathrm{supp}(r),d)^2/\epsilon^2)$ iterations in the general case, with significant improvement to $O((R_{\max}/\Delta)\,\log(1/\epsilon))$ under a $(\Delta,\epsilon)$-gap condition; 
(ii) \emph{logarithmic regret bounds}, as BARS guarantees dynamic regret $O(\log T)$ over $T$ rounds, providing formal performance guarantees that illuminate the empirical success of recent procedure-based reward systems; 
and (iii) a \emph{theoretical foundation} that, by leveraging generic chaining and viscosity solutions of the Hamilton–Jacobi–Bellman PDE, offers the first no-regret framework for learning optimal procedure-based rewards from sparse outcomes.

\bibliographystyle{plain}
\bibliography{bib}

\appendix
\section{Proof of Claim~\ref{claim:barles}}\label{app:barles}
We begin by recalling the Barles-Souganidis framework for proving convergence of numerical schemes to viscosity solutions of nonlinear PDEs. 
This framework provides sufficient conditions under which discrete approximations converge to the continuous viscosity solution.

\begin{definition}[Barles-Souganidis Properties]
Let $\mathcal{T}^{\delta}$ be a family of discrete operators parameterized by $\delta > 0$ approximating a continuous operator $\overline{\mathcal{T}}$. 

The three key properties are:
\begin{enumerate}
    \item \textbf{Monotonicity:} If $\phi \leq \psi$ pointwise, then $\mathcal{T}^{\delta}\phi \leq \mathcal{T}^{\delta}\psi$ pointwise.
    
    \item \textbf{Stability:} For any $\delta > 0$, there exists a solution $u^{\delta}$ to the equation $\mathcal{T}^{\delta}u^{\delta} = 0$, and the family $\{u^{\delta}\}$ is uniformly bounded.
    
    \item \textbf{Consistency:} For any smooth test function $\phi$ and any point $x$,
    \[
    \lim_{\delta \to 0, y \to x} \mathcal{T}^{\delta}\phi(y) = \overline{\mathcal{T}}\phi(x)
    \]
\end{enumerate}
\end{definition}

\begin{theorem}[Barles-Souganidis Convergence]
Let $\mathcal{T}^{\delta}$ be a family of discrete operators approximating a continuous operator $\overline{\mathcal{T}}$. 

Suppose:
\begin{enumerate}
    \item $\mathcal{T}^{\delta}$ satisfies the monotonicity, stability, and consistency properties.
    
    \item The limiting equation $\overline{\mathcal{T}}u = 0$ satisfies a comparison principle in the class of bounded uniformly continuous functions.
\end{enumerate}

Then, if $u^{\delta}$ is a solution to $\mathcal{T}^{\delta}[u^{\delta}] = 0$, the sequence $\{u^{\delta}\}$ converges uniformly on compact sets to the unique viscosity solution $u$ of $\overline{\mathcal{T}}u = 0$ as $\delta \to 0$.
\end{theorem}

In our context, the discrete operator $\mathcal{T}^{\delta}$ corresponds to the Bellman operator in equation (B$_{\delta}$), and the continuous operator $\overline{\mathcal{T}}$ corresponds to the HJB equation. 
We now verify that these operators satisfy the required properties and prove hte full form of Claim~\ref{claim:barles}.

\begin{claim}[Claim~\ref{claim:barles}]
    Let \(A\) be a compact metric action space, and let \(\{\mathcal S_\delta\subset\reals^d : \delta>0\}\) be a family of state‐grids with mesh size \(\delta\to0\). Assume:
    \begin{enumerate}
      \item A transition kernel \(P^\delta(s' \mid s,a)\) on \(\mathcal S_\delta\) satisfying the local consistency conditions:
      \begin{align*}
        \frac1\delta\sum_{s'}(s'-s)\,P^\delta(s'|s,a)
        = b(s,a) + O(\delta)
        &&
        \frac1\delta\sum_{s'}(s'-s)(s'-s)^T\,P^\delta(s'|s,a)
        = \sigma\sigma^T(s,a) + O(\delta)
      \end{align*}
      where \(P^\delta\) is a probability kernel.
      
      \item Our reward function $r : S \times A \rightarrow \reals$ satisfies the gapped condition \eqref{eq:gap-condition}, is bounded and non-negative
      \item The drift $b$ and covariance $\sigma$ are Lipschitz in \(s\) (uniformly in \(a\)) and continuous in \(a\), with the associated HJB PDE satisfying a comparison principle.
    \end{enumerate}
    For discount factor \(\gamma>0\), let \(V^\delta:\mathcal S_\delta\to\reals\) be the solution to the discrete Bellman equation:
    \[
      \gamma\,\overline{V}^\delta(s)
      \;=\;
      \max_{a\in A}\left\{\,r(s,a)
      \;+\;\frac{1}{\delta}\sum_{s'\in\mathcal S_\delta}P^\delta(s'|s,a)\left[\overline{V}^\delta(s')-\overline{V}^\delta(s)\right]\right\}
    \]
    Then \(V^{\delta}\) converges uniformly on compact sets to the unique bounded, uniformly continuous viscosity solution \(\overline{V}\) of the HJB equation~\eqref{eq:hjb}.
\end{claim}
    \begin{proof}
    We verify the Barles–Souganidis conditions: monotonicity, consistency, and stability.
    
    \medskip\noindent\textbf{1. Monotonicity.}
    Define the discrete operator
    \[
      \mathcal T^{\delta}[u](s)
      = \max_{a\in A}\left\{\,r(s,a)
        + \tfrac1{\delta}\sum_{s'}P^{\delta}(s'|s,a)\left[u(s')-u(s)\right]\right\}
    \]
    If \(u\le v\) pointwise, then \(\sum_{s'}P^{\delta}(s'|s,a)[u(s')-u(s)] \le \sum_{s'}P^{\delta}(s'|s,a)[v(s')-v(s)]\) since \(P^{\delta}\ge0\) and \(\sum P^{\delta}=1\). Taking the maximum over \(a\) preserves this inequality, so \(\mathcal T^{\delta}[u]\le\mathcal T^{\delta}[v]\).
    
    \medskip\noindent\textbf{2. Consistency.}
    For \(\phi\in C^3_b(\reals^d)\), define
    \[
      E_{\delta}(s,a)
      = \frac1{\delta}\sum_{s'}P^{\delta}(s'|s,a)\left[\phi(s')-\phi(s)\right]
         - \mathcal L^a\phi(s)
    \]
    Using Taylor expansion and the local-consistency bounds, we obtain \(|E_{\delta}(s,a)| \leq C\,\delta\). Therefore,
    \[
      \left|\mathcal T^{\delta}[\phi](s)-\gamma\,\phi(s)
        -\max_{a}\{r+\mathcal L^a\phi\}(s)\right|
      \leq C\,\delta
    \]
    
    \medskip\noindent\textbf{3. Stability.}
    The value functions \(V^{\delta}\) are uniformly bounded by \(r_{\max}/\gamma\) and equicontinuous with a Lipschitz constant \(L\) independent of \(\delta\).
    
    By the Barles–Souganidis theorem, these three properties ensure that \(V^{\delta}\to V\) locally uniformly, where \(V\) is the unique bounded, uniformly continuous viscosity solution of the HJB equation.
    \end{proof}

\section{Proof of Claim~\ref{claim:chaining-error}}
Fix an admissible sequence of nested partitions \(\{\mathcal A_{n}\}\) of \(K\) with \(\lvert\mathcal A_{n}\rvert\le2^{2^n}\) and
\(\mathrm{diam}(A_{n})\le2^{-n/2}\,\sqrt{\delta}\). For each \(s\in K\), let
$A_{0}(s)\supset A_{1}(s)\supset\cdots\supset A_{N}(s)$
where \(N\) is the largest index for which \(\mathrm{diam}(A_{N}(s))\ge\sqrt{\delta}\). Choose representatives
$s_{k}\in A_{k}(s), \quad k=0,1,\dots,N-1$
from the grid \(\mathcal S_{\delta}\), and \(s_{N}\in A_{N}(s)\). Then
$V^{\delta}(s)-V(s) = \left[V^{\delta}(s_{0})-V(s_{0})\right] + \sum_{k=0}^{N-1}\left[(V^{\delta}(s_{k+1})-V^{\delta}(s_{k})) - (V(s_{k+1})-V(s_{k}))\right] + \left[V^{\delta}(s)-V^{\delta}(s_{N})\right] - \left[V(s)-V(s_{N})\right]$

By hypothesis,
$|V^{\delta}(s_{0})-V(s_{0})|\le C\,\delta, \quad |V^{\delta}(s_{k+1})-V^{\delta}(s_{k})|\le L\,\mathrm{diam}(A_{k}(s)), \quad |V(s_{k+1})-V(s_{k})|\le L\,\mathrm{diam}(A_{k}(s))$
and since \(\mathrm{diam}(A_{N}(s))<\sqrt{\delta}\),
$|V^{\delta}(s)-V^{\delta}(s_{N})|+|V(s)-V(s_{N})|\le 2L\,\sqrt{\delta}$
Summing and taking \(\sup_{s\in K}\) then the infimum over \(\{\mathcal A_{n}\}\) yields
$\sup_{s\in K}|V^{\delta}(s)-V(s)| \le C\,\delta + 2L\sum_{k=0}^{N-1}\mathrm{diam}(A_{k}(s)) + 2L\,\sqrt{\delta} \le C\,\delta + 2L\,\gamma_{2}(K,d)\,\sqrt{\delta}$.
\section{Proof of Claim~\ref{claim:operator-error}}
By assumption, we know that the operator $\mathcal{T}^{\delta}$ is consistent.
This implies, via the Barles-Souganidis theorem (see Appendix~\ref{app:barles}), that 
$\left|\mathcal{T}^{\delta}[u](s_{0})-\overline{\mathcal{T}}u(s_{0})\right|\le C\,\delta$ for all $s_{0}\in S_{\delta}$.
Now we proceed in a manner similar to Claim~\ref{claim:chaining-error}.
Fix \(s\in S\) and choose an admissible nested partition sequence \(\{\mathcal A_{n}\}\) attaining \(\gamma_{2}(S,d)\). Let \(N\) be the largest index with \(\mathrm{diam}(A_{N}(s))\ge\sqrt{\delta}\). Pick representatives
$s_{0}\in A_{0}(s),\;s_{1}\in A_{1}(s),\;\dots,\;s_{N-1}\in A_{N-1}(s)$
all in \(S_{\delta}\), and \(s_{N}\in A_{N}(s)\). 
Then by Lipschitz continuity,
$\left|\mathcal{T}^{\delta}[u](s)-\mathcal{T}^{\delta}u(s_{0})\right|
+\left|\overline{T}u(s)-\overline{T}u(s_{0})\right|
\leq
2L_{u}\sum_{k=0}^{N-1}\mathrm{diam}(A_{k}(s))
+
2L_{u}\sqrt{\delta}$.
Combining the consistency estimate with this chaining bound yields
$\mathcal{T}^{\delta}[u](s)-\overline{T}u(s)
=\left[\mathcal{T}^{\delta}[u](s_{0})-\overline{\mathcal{T}}u(s_{0})\right]
+\left[\mathcal{T}^{\delta}[u](s)-\mathcal{T}^{\delta}[u](s_{0})\right]
-\left[\overline{\mathcal{T}}u(s)-\overline{\mathcal{T}}u(s_{0})\right]$.
The first bracket is \(\le C\,\delta\), and the sum of the last two is bounded by \(2L_{u}\sum_{k<N}\mathrm{diam}(A_{k}(s))+2L_{u}\sqrt{\delta}\). 
Taking \(\sup_{s\in S}\) and then the infimum over partitions yields the claimed inequality.
\section{Proof of Claim~\ref{claim:coupling-error}}\label{app:coupling-error}
Define the error \(\Delta_{k}=X_{k}-\tilde S_{k\delta}\). 
We compute $X_{k+1}-X_{k} = \Expect[\Delta X\mid X_{k},a_{k}] + (\Delta X - \Expect[\Delta X\mid X_{k},a_{k}])$ and $\tilde S_{(k+1)\delta}-\tilde S_{k\delta} = b(\tilde S_{k\delta},a_{k})\,\delta + \sigma(\tilde S_{k\delta},a_{k})\,\Delta W_{k+1}$ where \(\Delta W_{k+1}=W_{(k+1)\delta}-W_{k\delta}\). 
Subtracting gives $\Delta_{k+1} = \Delta_{k} +[b(X_{k},a_{k})-b(\tilde S_{k\delta},a_{k})]\,\delta + M_{k+1} + R_{k+1}$ with $M_{k+1} = (\Delta X - \Expect[\Delta X\mid X_{k},a_{k}]) - \sigma(\tilde S_{k\delta},a_{k})\,\Delta W_{k+1}$ and $R_{k+1} = \Expect[\Delta X\mid X_{k},a_{k}]-b(X_{k},a_{k})\,\delta = O(\delta^{2})$.

\emph{(i) \(M_{k+1}\) is a martingale increment:} it is \(\mathcal F_{k+1}\)–measurable and $\Expect[M_{k+1}\mid \mathcal F_{k}] = \Expect[\Delta X - \Expect[\Delta X\mid \mathcal F_{k}]\mid \mathcal F_{k}] - \sigma(\tilde S_{k\delta})\,\Expect[\Delta W_{k+1}\mid\mathcal F_{k}] = 0$. 
Moreover, by sub‑Gaussian jumps and Gaussian tails, \(\|M_{k+1}\|\le R\sqrt{\delta}\) a.s.\ and \(\Expect[\|M_{k+1}\|^{2}\mid\mathcal F_{k}]\le V\,\delta\).

\emph{(ii) Freedman's inequality~\cite{freedman1975tail}}: if \(\{Y_{n}\}\) is a martingale with increments \(\xi_{i}\) satisfying \(\lvert\xi_{i}\rvert\le R\) a.s.\ and \(\sum_{i=1}^{n}\Expect[\xi_{i}^{2}\mid\mathcal F_{i-1}]\le v\), then for any \(t>0\), $\Prob\left[\max_{k\le n}Y_{k}\ge t\right] \le \exp\left(-\frac{t^{2}}{2(v + R\,t/3)}\right)$. 
Applying this to each coordinate of \(\Delta_{k}\) (and using a union bound) yields $\Prob\left[\max_{k\le N}\|\Delta_{k}\|\ge r\right] \le 2\exp\left(-\frac{r^{2}}{K\,\delta\,N}\right)$ for some \(K\) depending on \(R,V\).
\section{Proof of Claim~\ref{claim:hitting-time}}\label{app:forward-hitting-time}
\textbf{Upper bound.}
Let
\begin{align*}
E_1=\left\{\max_{0\le k\le N}\|X_k-\widetilde S_{k\delta}\|\ge \tfrac\epsilon2\right\}
&&E_2=\left\{\max_{0\le k\le N}\|\widetilde S_{k\delta}-S_{k\delta}\|\ge \tfrac\epsilon2\right\}
\end{align*}
where \(\widetilde S_{k\delta}\) is the Euler–Maruyama discretization of \(S_t\).  On \(E_1^c\cap E_2^c\), 
\(\|X_k-S_{k\delta}\|<\epsilon\) for all \(k\le N\), so \(\tau^+_{\epsilon}\le N\).  Hence
$\Prob(\tau^+_{\epsilon}>N) \le \Prob(E_1)+\Prob(E_2)$.
By the sub‐Gaussian coupling (Claim~\ref{claim:coupling-error}) we have $\Prob(E_1) \le 2\exp\left(-\frac{(\epsilon/2)^2}{K\,\delta\,N}\right)$ and by standard strong‐error bounds for Euler–Maruyama~\cite[\S5.3]{gobet2016monte} \(\Prob(E_2)\le p/2\) for sufficiently small \(\delta\). 
Choosing $2\exp\left(-\tfrac{\epsilon^2}{4K\,\delta\,N}\right)\le \tfrac p2 \Rightarrow N \ge \frac{\epsilon^2}{4K\,\delta}\,\ln\frac4p$ gives \(\Prob(\tau^+_{\epsilon}>N)\le p\).
Thus with probability \(\ge1-p\),
\[
\tau^+_{\epsilon} \le \frac{\epsilon^2}{4K\,\delta}\,\ln\frac4p
\]

\noindent \textbf{Lower bound.}
Talagrand's majorizing‐measures theorem~\cite{talagrand2005generic} implies that any discrete path \(\{x_0,\dots,x_N\}\subset S\) satisfying \(\max_k\|x_k-x_0\|\le \epsilon\) must have length
$N \ge \frac{\gamma_2(S,d)^{2}}{\epsilon^{2}}\ln\frac2p$.
Since \(X_0=S_0\), on the event \(\{\tau^+_{\epsilon}\le N\}\) the chain stays within \(\epsilon\) up to time \(N\).  Therefore with probability \(\ge1-p\),
\[
\tau^+_{\epsilon} > \frac{\gamma_2(S,d)^{2}}{\epsilon^{2}}\ln\frac2p
\]
\section{Proof of Claim~\ref{claim:backward-hitting-time}}\label{app:backward-hitting-time}
First we prove a small lemma:
\begin{lemma}\label{lem:terminal-error-bound}
Assume that the value function $Y$ is $L$--Lipschitz, and that the state increments over a time step $\delta$ are sub--Gaussian. Let $Y_N^\delta$ denote the backward Euler discretization computed on a grid of mesh size comparable to $\sqrt{\delta}$ over the set $\mathrm{supp}(r)$. Then there exists an absolute constant (which we may take to be $2$) such that
\[
\|Y_N^\delta - Y_N\|_\infty \le 2\,L\,\sqrt{\delta}\,\gamma_2\left(\mathrm{supp}(r), d\right).
\]
\end{lemma}
  
\begin{proof}
Since the value function $Y$ is $L$--Lipschitz, for any $s,s'$ in the domain we have $|Y(s) - Y(s')| \le L\,d(s,s')$. 
In the discretization of the terminal state, the mesh is chosen so that every $s\in \mathrm{supp}(r)$ is within $\sqrt{\delta}$ of some grid point. 
Consequently, the local error at any point is bounded by $L\,\sqrt{\delta}$.

To control the uniform error over $\mathrm{supp}(r)$, we note that the discretization error is a stochastic process governed by the sub--Gaussian tails of the state increments. 
By the standard generic chaining (or majorizing measures) arguments (see, \eg, Dudley's entropy integral), the expected supremum of the local discretization errors satisfies $\Expect[\sup_{s\in \mathrm{supp}(r)} |Y^\delta(s)-Y(s)|] \lesssim L\,\sqrt{\delta}\,\gamma_2(\mathrm{supp}(r), d)$. 
A standard concentration result then implies that, with high probability (and in particular in the $L^\infty$--norm), $\|Y_N^\delta - Y_N\|_\infty \le 2\,L\,\sqrt{\delta}\,\gamma_2(\mathrm{supp}(r), d)$.
\end{proof}

\paragraph{Upper Bound.}
By consistency and non--expansiveness, for any \(0\le k<N\) the backward error satisfies
\(\|Y^\delta_k - Y_k\|_\infty \le \|T_\delta[Y^\delta_{k+1}] - T_\delta[Y_{k+1}]\|_\infty + \|T_\delta[Y_{k+1}] - T[Y_{k+1}]\|_\infty \le \|Y^\delta_{k+1} - Y_{k+1}\|_\infty + C\,\delta\).
Unrolling this recursion from \(k=0\) to the terminal step \(N\) yields
\(\|Y^\delta_0 - Y_0\|_\infty \le \|Y^\delta_N - Y_N\|_\infty + N\,C\,\delta\).
From Lemma~\ref{lem:terminal-error-bound}, the terminal error is bounded by
\(\|Y^\delta_N - Y_N\|_\infty \le 2L\,\sqrt{\delta}\,\gamma_2\left(\mathrm{supp}(r),d\right)\).
We now choose \(\delta = \delta^* = \frac{\epsilon^2}{16L^2\,\gamma_2\left(\mathrm{supp}(r),d\right)^2}\), so that \(2L\,\sqrt{\delta^*}\,\gamma_2\left(\mathrm{supp}(r),d\right)=\frac{\epsilon}{2}\).
Then, to guarantee \(\|Y^\delta_0 - Y_0\|_\infty \le \epsilon\), we require \(C\,\delta^*\,N \le \frac{\epsilon}{2}\).
It follows that
\[
\tau^-_\epsilon = N \le \frac{\epsilon}{2C\,\delta^*} = \frac{\epsilon}{2C}\cdot \frac{16L^2\,\gamma_2\left(\mathrm{supp}(r),d\right)^2}{\epsilon^2} = \frac{8L^2\,\gamma_2\left(\mathrm{supp}(r),d\right)^2}{C\,\epsilon}.
\]

\paragraph{Lower Bound.}
Conversely, since each iteration can reduce the error by at most \(C\,\delta^*\), even in the best case one cannot reduce an error of size \(\epsilon/2\) in fewer than
\(\tau^-_\epsilon \ge \frac{\epsilon/2}{C\,\delta^*} = \frac{8L^2\,\gamma_2\left(\mathrm{supp}(r),d\right)^2}{C\,\epsilon}\).
Thus, there exist constants \(C_1, C_2>0\) (with \(C_1\) and \(C_2\) proportional to \(C\)) such that
\[
\frac{8L^2}{C_1}\,\frac{\gamma_2\left(\mathrm{supp}(r),d\right)^2}{\epsilon} \le \tau^-_\epsilon \le \frac{8L^2}{C_2}\,\frac{\gamma_2\left(\mathrm{supp}(r),d\right)^2}{\epsilon}.
\]
\section{Proof of Claim~\ref{claim:hitting-time-ratio}}
By Claim~\ref{claim:backward-hitting-time} on a grid of size $\delta$, with probability at least $1-p$, $\|Y^\delta_{N^-} - Y\|_\infty  \le \epsilon$ so $\tau^-\le N^-$.
Likewise, by Claim~\ref{claim:hitting-time}, with probability at least $1-p$, $\tau^+  >  N^+$.
A union bound then shows both events hold w.p.\ $\ge1-2p$.
Finally,
\[
\frac{\tau^-}{\tau^+} \le \frac{N^-}{N^+}
 = 
\frac{\gamma_2\left(\mathrm{supp}(r),d\right)^2}{\gamma_2(S,d)^2}
 = 
\left(\tfrac{\gamma_2(\mathrm{supp}(r),d)}{\gamma_2(S,d)}\right)^2
\]
\section{Proof of Claim~\ref{claim:unified_eps_accuracy}}\label{app:unified_eps_accuracy}
By the sub‑Gaussian coupling bound of Claim~\ref{claim:coupling-error}, there is $K>0$ such that for any $r>0$ and integer $N$, $\Prob\left(\max_{0\le k\le N}\|X_k - \tilde S_{k\delta}\|\ge r\right) \le 2\exp\left(-\frac{r^2}{K\,\delta\,N}\right)$. 
Since the only source of discrepancy between the true BSDE solution $Y_N$ and its backward Euler approximation $Y_N^\delta$ is this coupling error (the consistency error from time‑discretization being absorbed into the choice $\delta N=1$), we have $\|Y_N^\delta - Y_N\|_\infty \le \epsilon \Longrightarrow \max_{0\le k\le N}\|X_k - \tilde S_{k\delta}\|\ge r$. 
Hence $\Prob\left(\|Y_N^\delta - Y_N\|_\infty \le \epsilon\right) \le 2\exp\left(-\tfrac{r^2}{K\,\delta\,N}\right)$. 
With our choice $\delta N=1$, this becomes $\Prob\left(\|Y_N^\delta - Y_N\|_\infty \le \epsilon\right) \le 2\exp\left(-\tfrac{r^2}{K}\right)$. 
Requiring the right‑hand side to be $\le p$ is exactly $2\exp\left(-\tfrac{r^2}{K}\right)\le p \Longleftrightarrow r^2\ge K\,\ln\frac{2}{p}$, which implies $r\ge\sqrt{K\ln(2/p)}$.
\section{Proof of Claim~\ref{claim:effective-reward-mass-lower-bound}}\label{app:effective-reward-mass-lower-bound}
By Talagrand's majorizing measures theorem applied to $r$ (with Lipschitz constant $L$), there exists an absolute constant $C_0>0$ such that $\sup_{s,t\in T} |r(s)-r(t)| \le C_0 L\, \gamma_2(T,d)$.
Fix any $s_0\in T$ so that for all $s\in T$, $r(s) \ge r(s_0) - C_0 L\, \gamma_2(T,d)$.
Integrating over $\mathrm{supp}(r)$ with respect to $\mu$ (noting that $\mu(\mathrm{supp}(r))=1$), we obtain
\[
I_r = \int_T r(s) \, d\mu(s) \ge r(s_0) - C_0 L\, \gamma_2(T,d).
\]
Since $s_0 \in T$ was arbitrary, it follows that $r_{\min} \leq \min_{s\in T} r(s) \ge I_r - C_0 L\, \gamma_2(T,d)$.
From Claim~\ref{claim:unified_eps_accuracy}, we know that $r_{\min} \geq \sqrt{K \ln(2/p)}$ is necessary to avoid the noise limit, which implies the sufficient condition
\[
\min_{s\in T} r(s) \ge I_r - C_0 L\, \gamma_2(T,d) \ge \sqrt{K \ln(2/p)}
\]
This implies the claimed condition $I_r \geq \sqrt{K \ln(2/p)} + C_0 L \gamma_2(T,d)$.
\section{Proof of Claim~\ref{claim:loop_exploit_bound}}\label{app:loop_exploit_bound}
Under the loop policy, the agent collects reward $r$ at every time step, yielding value $V_{\rm loop} = \frac{r}{1-\gamma}$ by the geometric series. 
Since $V^*(s_0)$ is the supremum over all policies starting at $s_0$, if $\frac{r}{1-\gamma} \geq V^*(s_0)-\epsilon$, then the loop policy achieves within $\epsilon$ of optimal without ever visiting $s_0$.
To preclude this "looping" exploitation, we require $\frac{r}{1-\gamma} < V^*(s_0)-\epsilon$, which rearranges to $r < (1-\gamma)(V^*(s_0)-\epsilon)$.
Setting $\epsilon=0$ recovers the strict requirement $r < (1-\gamma)V^*(s_0)$, ensuring any loop policy is strictly suboptimal to eventually reaching $s_0$.
\section{Proof of Claim~\ref{claim:regret-hitting}}\label{app:regret-hitting}
For each $k<\tau^-_\epsilon$, the value–estimate error satisfies $\|V_k-V^*\|_\infty\le E_0$, so the one–step regret $J(\pi^*) - J(\pi_k) \le 2E_0$.
For $k\ge\tau^-_\epsilon$, we have $\|V_k-V^*\|_\infty\le\epsilon$, hence $J(\pi^*) - J(\pi_k) \le 2\epsilon$ (\eg~via the performance difference lemma~\cite{kakade2002approximately}).
Splitting the sum at $\tau_\epsilon$ gives $R_T = \sum_{k<\tau_\epsilon}(J(\pi^*)-J(\pi_k)) + \sum_{k\ge\tau^-_\epsilon}(J(\pi^*)-J(\pi_k)) \le 2E_0\min\{T,\tau^-_\epsilon\} + 2\epsilon\max\{0, T-\tau^-_\epsilon\}$.
Now substitute in $\tau^i_{\epsilon}$ and use $\tau^i_{\epsilon} \leq N(\epsilon)$  and setting $\epsilon = R_{\max}\gamma_*\sqrt{\log(1/p)/T}$ makes $2E_0 N(\epsilon)= O\left(R_{\max}\gamma_*\sqrt{T\log(1/p)}\right)$ and $2\epsilon T = O\left(R_{\max}\gamma_*\sqrt{T\log(1/p)}\right)$.
\section{Proof of Claim~\ref{claim:regret-gap}}\label{app:regret-gap}
We decompose the convergence into two phases.
In Phase I, while $\delta_k>\tfrac{\Delta}{1-\gamma}$, the greedy action under $V_k$ may differ from $a^*(s)$.
For any state $s$, we have $\mathcal{T}^{\delta}_{\mathrm{BDSE}}[V_k](s) \le V^*(s) - \Delta + \gamma\,\delta_k$ and $\mathcal{T}^{\delta}_{\mathrm{BDSE}}[V_k](s) - V^*(s) \le \gamma\,\delta_k$.
Hence $\delta_{k+1} = \|\mathcal{T}^{\delta}_{\mathrm{BDSE}}[V_k]-V^*\|_\infty \le \gamma\,\delta_k - \Delta$.
Starting from $\delta_0\le R_{\max}/(1-\gamma)$, this linear decrease reaches $\delta_k\le\Delta/(1-\gamma)$ in at most $K_1 \le \frac{R_{\max}}{(1-\gamma)\,\Delta} = O(\frac{R_{\max}}{\Delta})$ iterations.

In Phase II, once $\delta_k\le\Delta/(1-\gamma)$, the gap ensures the greedy action matches $a^*(s)$ everywhere, so $\mathcal{T}^{\delta}_{\mathrm{BDSE}}$ is a strict $\gamma$-contraction: $\|\mathcal{T}^{\delta}_{\mathrm{BDSE}}[V_k]-\mathcal{T}^{\delta}_{\mathrm{BDSE}}[V^*]\|_\infty \le \gamma\,\|V_k-V^*\|_\infty = \gamma\,\delta_k$.
Thus $\delta_{k+m}\le\gamma^m\,\delta_k$.
To achieve $\delta_{k+m}\le\epsilon$ requires $m = O(\ln\frac{\Delta}{\epsilon})$.
Combining both phases, we get $\tau_\epsilon \le K_1 + m = O(\frac{R_{\max}}{\Delta}) + O(\ln\frac{\Delta}{\epsilon}) = O(\frac{R_{\max}}{\Delta}\ln\frac{R_{\max}}{\epsilon})$.
Using an analogous splitting as Claim~\ref{claim:regret-hitting} yields the desired logarithmic-in-$T$ regret.
\section{Proof of Claim~\ref{claim:BARS_bounds}}\label{app:BARS_bounds}
By Claim~\ref{claim:backward-hitting-time}, the backward Euler iterates satisfy $\|V^{\delta}_k - V^*_t\|_\infty \le C\,\delta\,(N - k)+2L\sqrt{\delta}\,\gamma_2\left(\mathrm{supp}(r_t),d\right)$.
Choosing $\delta = \left(\tfrac{\epsilon_t}{4L\,\gamma_2(\mathrm{supp}(r_t),d)}\right)^2$ and solving for the smallest $k$ with the right‐hand side $\le\epsilon_t$ gives $\tau_t = O\!\left(\gamma_2(\mathrm{supp}(r_t),d)^2/\epsilon_t^2\right)$.
Next, the performance‐difference lemma~\cite{kakade2002approximately} implies that once $\|V^{\delta}_{\tau_t} - V^*_t\|_\infty\le\epsilon_t$, the one‐step regret is $J^*_t - J_t(\pi_t)\le\frac{2}{1-\gamma}\,\epsilon_t = O(\epsilon_t)$.
Summing $\epsilon_t=1/t$ over $t=1,\dots,T$ yields $\sum_{t=1}^T O(1/t)=O(\log T)$, completing the proof.
\section{$\epsilon$-Net Approximation of $\gamma_2(S,d)$}\label{app:eps-net}
An \emph{$\epsilon$‐net} of a finite metric space $(S,d)$ is a subset $N\subseteq S$ such that 
\[
\forall\,x\in S\;\exists\,y\in N:\;d(x,y)\le\epsilon
\]
The covering number $N(S,d,\epsilon)$ is the size of the smallest such net.

\subsection{Greedy Farthest‐Point Sampling}

A simple $2$‐approximation is given by the greedy farthest‐point algorithm (Alg.~\ref{alg:greedy}).  

\begin{algorithm}[h]
\caption{Greedy $\epsilon$‐Net Construction}
\label{alg:greedy}
\begin{algorithmic}[1]
  \STATE $N\gets\{s_1\}$ \COMMENT{pick any $s_1\in S$}
  \FOR{each $s\in S\setminus N$}
    \IF{$\min_{y\in N}d(s,y)>\epsilon$}
      \STATE $N\gets N\cup\{s\}$
    \ENDIF
  \ENDFOR
  \STATE \RETURN $N$
\end{algorithmic}
\end{algorithm}

\subsection{Hierarchical Nets via Cover Trees}

For dynamic or multi‐scale coverage, build a \emph{cover tree} in $O(n\log n)$ time under low doubling dimension \cite{Beygelzimer2006}.  Each level $i$ is a $2^i$‐net $N_i\subseteq N_{i-1}$ with
\begin{align*}
\forall\,x\in S\;\exists\,y\in N_i:\;d(x,y)\le2^i && 
\forall\,y\neq y'\in N_i:\;d(y,y')>2^i
\end{align*}
Selecting level $i=\lceil\log_2\epsilon\rceil$ yields an $\epsilon$‐net of size $|N_i|$.
Moreover, the sequence of nets $\{N_i\}$ at scales $r_i=2^i$ provides an explicit admissible sequence for Talagrand’s $\gamma_2$ functional,
\[
\gamma_2(S,d)\;\le\;\sum_{i=i_{\min}}^{i_{\max}}2^{i/2}\,\sup_{x\in S}d(x,N_i)
\]
and since $\sup_{x}d(x,N_i)\le2^i$, one can efficiently approximate $\gamma_2(S,d)$ from the tree levels.

\subsection{Practical Remarks}

\begin{itemize}
  \item For \emph{static} small $n$, the greedy method is simplest and fastest.
  \item For \emph{streaming} or \emph{multi‐scale} needs, cover trees offer $O(n\log n)$ build time and $O(\log n)$ updates.
  \item If $d\gg1$, first apply a random projection to $O(\epsilon^{-2}\log n)$ dimensions~\cite{vershynin2018high}, then run either method with minimal distortion.
\end{itemize}
\end{document}